\def\isarxiv{1} 

\ifdefined\isarxiv
\documentclass[11pt]{article}

\usepackage[numbers]{natbib}

\else
\documentclass{article}

\usepackage[final]{cpal_2025}

\fi

\usepackage{amsmath}
\usepackage{amsthm}
\usepackage{amssymb}
\usepackage{algorithm}
\usepackage{subfig}
\usepackage{algpseudocode}
\usepackage{graphicx}
\usepackage{grffile}
\usepackage{wrapfig,epsfig}
\usepackage{url}
\usepackage{xcolor}
\usepackage{epstopdf}

\usepackage{bbm}
\usepackage{dsfont}

\usepackage{enumitem} 
\usepackage{multirow} 
\usepackage{booktabs} 

\allowdisplaybreaks

\ifdefined\isarxiv
\renewcommand*{\citet}{\cite}
\renewcommand*{\citep}{\cite}

\usepackage{tikz}
\usepackage{hyperref}  
\hypersetup{colorlinks=true,citecolor=blue,linkcolor=blue} 
\usetikzlibrary{arrows}
\usepackage[margin=1in]{geometry}

\else

\usepackage{hyperref}

\fi
\graphicspath{{./figs/}}

\theoremstyle{plain}
\newtheorem{theorem}{Theorem}[section]
\newtheorem{lemma}[theorem]{Lemma}
\newtheorem{definition}[theorem]{Definition}

\newtheorem{corollary}[theorem]{Corollary}

\newtheorem{fact}[theorem]{Fact}
\newtheorem{remark}[theorem]{Remark}

\newcommand{\wh}{\widehat}
\newcommand{\wt}{\widetilde}
\newcommand{\ov}{\overline}
\newcommand{\N}{\mathcal{N}}
\newcommand{\R}{\mathbb{R}}

\newcommand{\Tinit}{{\cal T}_{\mathsf{init}}}
\newcommand{\Tquery}{{\cal T}_{\mathsf{query}}}
\newcommand{\Tupdate}{{\cal T}_{\mathsf{update}}}

\renewcommand{\d}{\mathrm{d}}

\DeclareMathOperator*{\E}{{\mathbb{E}}}
\DeclareMathOperator*{\var}{\mathrm{Var}}

\DeclareMathOperator{\diag}{diag}

\DeclareMathOperator{\sgn}{sgn}

\makeatletter
\newcommand*{\RN}[1]{\expandafter\@slowromancap\romannumeral #1@}
\makeatother

\usepackage{lineno}

\begin{document}

\ifdefined\isarxiv

\date{}

\title{HSR-Enhanced Sparse Attention Acceleration}

\author{
Bo Chen\thanks{\texttt{ bc7b@mtmail.mtsu.edu}. Middle Tennessee State University.}
\and 
Yingyu Liang\thanks{\texttt{
yingyul@hku.hk}. The University of Hong Kong. \texttt{
yliang@cs.wisc.edu}. University of Wisconsin-Madison.} 
\and
Zhizhou Sha\thanks{\texttt{ shazz20@mails.tsinghua.edu.cn}. Tsinghua University.}
\and
Zhenmei Shi\thanks{\texttt{
zhmeishi@cs.wisc.edu}. University of Wisconsin-Madison.}
\and 
Zhao Song\thanks{\texttt{ magic.linuxkde@gmail.com}. The Simons Institute for the Theory of Computing at the UC, Berkeley.}
}

\else

\title{\includegraphics[scale=0.042]{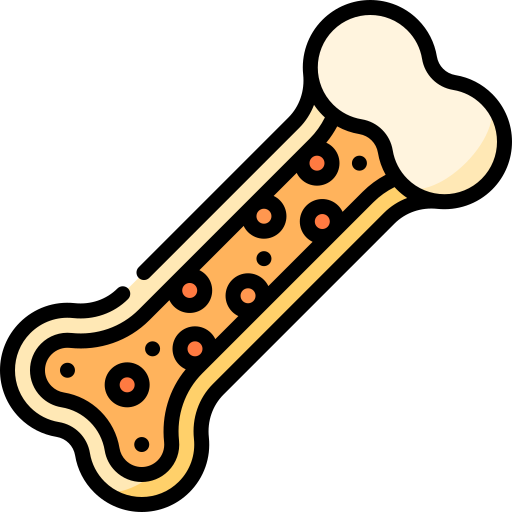} HSR-Enhanced Sparse Attention Acceleration}

\author{%
  Bo Chen\textsuperscript{1}, ~~~Yingyu Liang\textsuperscript{2,3} , ~~~Zhizhou Sha\textsuperscript{4} , ~~~Zhenmei Shi\textsuperscript{3} , ~~~Zhao Song\textsuperscript{5} 
  \\
    \textsuperscript{1}Middle Tennessee State University \quad \textsuperscript{2}The University of Hong Kong\\
     \textsuperscript{3}University of Wisconsin-Madison \quad
      \textsuperscript{4}Tsinghua University\\
       \textsuperscript{5}The Simons Institute for the Theory of Computing at the University of California,
Berkeley, \\
  \texttt{bc7b@mtmail.mtsu.edu, yingyul@hku.hk, shazz20@mails.tsinghua.edu.cn, zhmeishi@cs.wisc.edu, magic.linuxkde@gmail.com}
}

\fi

\ifdefined\isarxiv
\begin{titlepage}
  \maketitle
  \begin{abstract}
Large Language Models (LLMs) have demonstrated remarkable capabilities across various applications, but their performance on long-context tasks is often limited by the computational complexity of attention mechanisms. We introduce a novel approach to accelerate attention computation in LLMs, particularly for long-context scenarios. We leverage the inherent sparsity within attention mechanisms, both in conventional Softmax attention and ReLU attention (with $\mathsf{ReLU}^\alpha$ activation, $\alpha \in \mathbb{N}_+$), to significantly reduce the running time complexity. Our method employs a Half-Space Reporting (HSR) data structure to identify non-zero or ``massively activated'' entries in the attention matrix. We present theoretical analyses for two key scenarios: generation decoding and prompt prefilling. Our approach achieves a running time of $O(mn^{4/5})$ significantly faster than the naive approach $O(mn)$ for generation decoding, where $n$ is the context length, $m$ is the query length, and $d$ is the hidden dimension. We can also reduce the running time for prompt prefilling from $O(mn)$ to $O(mn^{1 - 1 / \lfloor d/2\rfloor} + mn^{4/5})$. Our method introduces only provably negligible error for Softmax attention. This work represents a significant step towards enabling efficient long-context processing in LLMs.

  \end{abstract}
  \thispagestyle{empty}
\end{titlepage}

{\hypersetup{linkcolor=black}
\tableofcontents
}
\newpage

\else
\maketitle 
\begin{abstract}
Large Language Models (LLMs) have demonstrated remarkable capabilities across various applications, but their performance on long-context tasks is often limited by the computational complexity of attention mechanisms. We introduce a novel approach to accelerate attention computation in LLMs, particularly for long-context scenarios. We leverage the inherent sparsity within attention mechanisms, both in conventional Softmax attention and ReLU attention (with $\mathsf{ReLU}^\alpha$ activation, $\alpha \in \mathbb{N}_+$), to significantly reduce the running time complexity. Our method employs a Half-Space Reporting (HSR) data structure to identify non-zero or ``massively activated'' entries in the attention matrix. We present theoretical analyses for two key scenarios: generation decoding and prompt prefilling. Our approach achieves a running time of $O(mn^{4/5})$ significantly faster than the naive approach $O(mn)$ for generation decoding, where $n$ is the context length, $m$ is the query length, and $d$ is the hidden dimension. We can also reduce the running time for prompt prefilling from $O(mn)$ to $O(mn^{1 - 1 / \lfloor d/2\rfloor} + mn^{4/5})$. Our method introduces only provably negligible error for Softmax attention. This work represents a significant step towards enabling efficient long-context processing in LLMs.
\end{abstract}

\fi

\section{Introduction}
Large Language Models (LLMs) have showcased remarkable capabilities across various applications, including context-aware question answering, content generation, summarization, and dialogue systems, among others \citep{tdh+22, cdi+21, wtb+22, zld+24}. 
Long-context tasks of LLMs have gained more and more attention. 
Several LLMs extend their context length to $128$K tokens, such as Yarn~\citep{pqj+23}, GPT-4~\citep{gpt4turbo}, Claude 3.5 \citep{claude3.5}, Llama 3.1 \citep{llama3}, Phi-3.5 \citep{phi3}, Mistral Nemo \citep{mistral_nemo}, etc. 
A bottleneck for long-context tasks is the computational cost of the attention mechanism in LLMs.
The key to LLM success is the transformer architecture \citep{vsp+17}, wildly used in various practical scenarios \citep{rwc+19, kjt19, wsd+23, wcz+23, wxz+24}, whose critical component is the attention mechanism. 
Let $n$ be the data length, $m$ be the length of query tokens, and $d$ be the feature dimension\footnote{As $d$ is always fixed in practice, there is no need to scale up $d$ in analysis. Thus, in this work, we always assume $d$ is a small constant.}. 
The conventional attention uses Softmax activation and is defined as follows: 
\begin{definition}[Softmax attention] \label{def:Softmax_attention}
Let $Q \in \R^{m \times d} $ and $K, V \in \R^{n \times d}$ denote the query, key, and value matrix.
The Softmax attention is:
\begin{align*}
    \mathsf{Attn}_s(Q,K,V) := \mathsf{Softmax}(QK^\top / \sqrt{d}) V = D^{-1} A_s V \in \R^{m \times d},
\end{align*}
where (1) $A_s := \exp( Q K^\top / \sqrt{d}) \in \R^{m \times n}$ and $\exp$ is applied element-wise , (2) $D := \diag (A_s \cdot {\bf 1}_n ) \in \R^{m \times m}$ denotes the normalization matrix, (3) $  D^{-1} A_s \in \R^{m\times n}$ denotes the attention matrix. 
\end{definition}

In practical LLM applications, there are two scenarios for attention computation depending on the context length $n$ and query length $m$. 
The first case, $m=\Theta(1)$, represents the generation decoding based on the pre-computed Key Value Cache (KV), which stores the intermediate attention key and value matrices.
The second case, $m=\Theta(n)$, represents the prompt prefilling before text generation or the cross-attention computation. 
However, in both cases, when the context window $n$ becomes larger, the running time will increase correspondingly, i.e., it will be linear and quadratic in $n$ for $m=\Theta(1)$ and $m=\Theta(n)$, respectively. 
Thus, reducing the running time of attention computations with long context input becomes essential to minimize response latency and increase throughput for LLMs.

In this work, we introduce novel methods to reduce the running time complexity for both cases, i.e., $m=\Theta(1)$ and $m=\Theta(n)$. 
We are inspired by the inherent sparsity within attention mechanisms. 
Numerous prior studies have highlighted the significant sparsity in the attention matrix \citep{cgrs19, apb+23, lwd+23, tzz+24, sckl24}. 
This manifestation of sparsity in Softmax attention is that a large number of attention scores, i.e., $QK^\top$, concentrate on a small number of entries, which is known as ``massive activation''. 
Due to this nature, Softmax attention can be accelerated by
only calculating the entries that contain large attention scores, introducing negligible approximation errors~\citep{zsz+23, lhy+24}. 

When talking about ReLU activation, one can easily accelerate the computation process by only calculating the entries activated by ReLU (since other non-activated entries will eventually be set to zero by ReLU).
ReLU attention is another attention mechanism widely used, substituting the conventional Softmax activation function with ReLU. ReLU attention has demonstrated performance comparable to Softmax attention in various downstream tasks \citep{wlgk23, hdll22} (see  Section~\ref{sec:related_work} for more details). We present the formal definition of ReLU attention as follows.
\begin{definition}[ReLU attention] \label{def:relu_attention}
Let $Q \in \R^{m \times d} $ and $K, V \in \R^{n \times d}$ denote the query, key, and value matrix. Let $\alpha \in \mathbb{N}_+$. 
The ReLU attention is:
\begin{align*}
    \mathsf{Attn}_r(Q,K,V) :=  D^{-1} A_r V \in \R^{m \times d},
\end{align*}
where (1) $A_r := \mathsf{ReLU}^{\alpha}( Q K^\top /\sqrt{d} - b) \in \R^{m \times n}$ and $\mathsf{ReLU}^{\alpha}$ denotes the $\alpha$-th power of ReLU activation for any $\alpha \in \mathbb{N}_+$, 
(2) $D := \diag (A_r \cdot {\bf 1}_n ) \in \R^{m \times m}$ denotes the normalization matrix, 
(3) $b \in \R$ denotes position bias, 
(4) $D^{-1} A_r \in \R^{m \times n}$
denotes the attention matrix. 
\end{definition}

\begin{wrapfigure}{r}{0.6\textwidth}
\centering
\includegraphics[width=\linewidth]{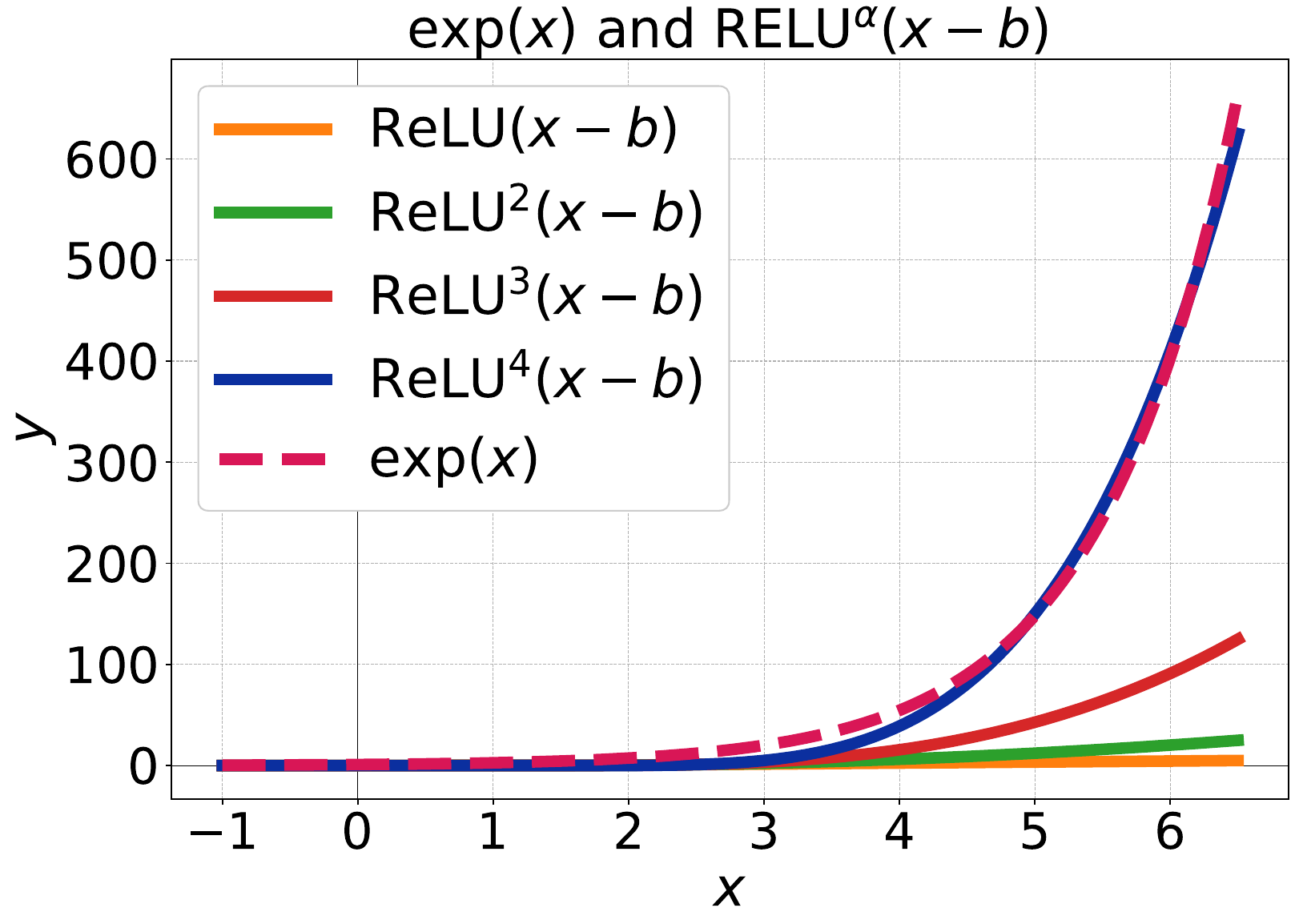}
\caption{
The trending of the Softmax activation ($\exp$) and the ReLU activation with different powers. 
Here, we choose $b = 1.5$ as the threshold for the ReLU activation. 
}
\label{fig:exp_relu_power}
\end{wrapfigure}

To expedite the computation, the critical task is to identify the large/non-zero entries for Softmax/ReLU attention, respectively. 
In this work, We utilize the half-space reporting (HSR) data structure to tackle this problem. HSR was first introduced in \cite{aem92} to address the half-space range reporting problem (More details can be found in Section~\ref{sec:preliminary:hsr}). 

In our framework, we define the half-space as the region where the attention scores (the inner products of key and query vectors) exceed some threshold. 
We leverage the HSR data structure’s ability to quickly answer range reporting to expedite the identification of non-zero entries within the ReLU attention matrix and large entries in Softmax attention. 
Consequently, we accelerate the computation for ReLU attention and expedite the computation of Softmax attention with negligible approximation error, resulting in a substantial reduction in computation time. 

Then, we state our results under the generation decoding scenario ($m=\Theta(1)$) and prompt prefilling scenario ($m=\Theta(n)$). 
When $m=\Theta(1)$, we accelerate ReLU and Softmax attention computation time over the naive approach from $O(mn)$ to $O(mn^{4/5})$ with pre-processed KV cache (Algorithm~\ref{alg:relu_attn_gen}). 
When $m=\Theta(n)$, we accelerate ReLU and Softmax attention computation time over the naive approach from $O(mn)$ to $O(mn^{1 - 1 / \lfloor d/2  \rfloor} + mn^{4/5})$ (Algorithm~\ref{alg:calculation_general_framework}). 
Furthermore, Section~\ref{sec:experiments} shows that the approximation error associated with Softmax attention utilizing ``massive activated'' entries only is small in practice, which is consistent with our theoretical analysis.

\paragraph{Our contributions:}
\begin{itemize}
    \item 
    To the best of our knowledge, this is the first work incorporating the HSR data structure with attention computation to reduce the running time complexity with the help of the sparsity within the attention mechanisms.
    \item We provide rigorous theoretical proofs for reducing the computational time (1) for ReLU attention generation decoding from $O(mn)$ to $O(mn^{4/5})$  (Algorithm~\ref{alg:relu_attn_gen}); (2) for ReLU attention prompt prefilling from $O(mn)$ to $O(mn^{1 - 1 / \lfloor d/2\rfloor} + mn^{4/5})$ (Algorithm~\ref{alg:calculation_general_framework}). 
    \item We achieve the same running time speed up for the conventional Softmax attention (Theorem~\ref{thm:Softmax_attention_generation:informal}, \ref{thm:Softmax_attention_computation:informal}), and we give rigorous theoretical proofs to ensure that the approximation error remains negligible (Theorem~\ref{thm:err_analysis_of_Softmax_attn_with_index_set:informal}). And we provide an empirical evaluation to prove our theoretical error analysis (Section~\ref{sec:experiments}). 
    \item We conduct empirical experiments on prominent LLMs to verify the approximation error associated with Softmax attention utilizing ``massive activated'' entries only. The results show that the error using a few top entries is already insignificant, consistent with our theoretical analysis.
\end{itemize}

\paragraph{Comparison with Previous Works. }
\cite{fa23} uses an approximated nearest neighbors search, which requires that the data are well-conditioned, e.g., uniformly separated, while our algorithm supports an exact nearest neighbor search. Their nearest neighbors search may introduce large approximation errors, while the approximation error of our algorithm can be small. On the other hand, \cite{as23} uses low-rank approximation methods to accelerate the attention computation while they require bounded entries assumptions, which is not required in this work. The bounded entry assumption may not always be held in practical scenarios. Our work is based on practical observation that the attention matrix is sparse.

\paragraph{Roadmap.} 
Section~\ref{sec:related_work} presents related work.
Section~\ref{sec:preliminary} introduces essential concepts.
Section~\ref{sec:main_result} presents our main results, i.e., guarantees on run time reduction and approximation error.
Section~\ref{sec:extension} introduces the extension of our method on prompt prefilling scenarios. 
Section~\ref{sec:tech_overview} provides a summary of the techniques used in our proof. 
Section~\ref{sec:experiments} provides our empirical evaluation for the approximation error on Softmax attention. 
Section~\ref{sec:discussion} discusses the potential of extending our method.
Section~\ref{sec:conclusion} concludes our algorithm and contributions.

\section{Related Work} \label{sec:related_work}

\subsection{Attention Acceleration for Long Context Input}
A long context window is essential for transformer-based LLMs in many downstream tasks. 
However, due to the quadratic time complexity associated with attention mechanisms, 
transformers are usually hard to run inference efficiently. 
Numerous methods have been proposed to enhance the inference efficiency. One approach involves using alternative architectures as proxies for attention to support faster inference, such as Mamba~\citep{gd23,dg24}, PolySketchFormer~\citep{kmz23}, Hopfield Models~\citep{hyw+23,whl+24,hlsl24,xhh+24,whhl24,hcl+24,hcw+24} and Linearizing Transformers~\citep{zbkr24,mvk+24}. 
Another line of research focuses on approximating attention matrix computation~\citep{as23,as24_arxiv,as24_iclr,hjk+23,zhmk24,lssz24a,pmn+23,ctwc24,lssy24,lls+24b,gswy23,dyz+24,lss+24,kll+25_var,chl+24_rope}. Nevertheless, these methods often rely on assumptions that may not be practical. For instance, some approaches use polynomial methods to approximate the exponential function, which requires all entries to be bounded by a small constant.
However, our HSR-enhanced attention framework is designed based on practical observation and validated by empirical support. 
These advancements not only improve general model performance but also play a crucial role in enhancing in-context learning capabilities, where models leverage information from the immediate context to perform tasks without fine-tuning. We refer the readers to some other related works~\citep{lss+24_relu, lls+24_io, lls+24_prune, lssz24_dp, lssz24a, lssz24_gm, swxl24, xsl24, smn+24, llss24_sparsegpt, syz23, qszz23, syyz23_dp, z22, szz24, syz23, ccl+25, cgl+25_homo, cll+25_icl, cll+24_rope, ssz+24_dit, ssz+24_pruning,cll+25,lll+25_loop,kls+25,lls+25_grok,wms+24,hwsl24,hwl+24,hwg+24,wsh+24}.

\subsection{ReLU Attention}
ReLU attention employs the ReLU activation function in place of the traditional Softmax function for attention computation. 
Previous studies have highlighted the promising potential of ReLU attention in various domains. 
From the empirical side, \cite{wlgk23} has demonstrated that incorporating ReLU as the activation function in vision transformers enhances performance on downstream tasks. \cite{sgt+23} has shown that transformers equipped with ReLU attention outperform those with Softmax attention, particularly when dealing with large key-value memory in machine translation tasks.
From the theoretical side, the scale-invariant property of ReLU attention \citep{lbz+22} facilitates the scalability of transformer networks. 
Furthermore, \cite{bcw+23, fgbm23, lss+25_relu} have shown that the inherent properties of ReLU attention contribute positively to the learning process of transformer models.
Another key advantage is that the ReLU function effectively sets all negative values to zero, allowing us to bypass these non-contributory elements during attention computation and thereby reducing its running time. 
Omitting these zero and negative entries does not introduce any error into the final output of the ReLU attention mechanism. 

\subsection{Half-Space Reporting (HSR) Data Structure}
The HSR data structure, initially proposed by \cite{aem92}, was developed to address the half-space range reporting problem.
The expedited range query capability inherent to HSR has been demonstrated to significantly enhance computational efficiency across a variety of tasks.
Studies such as \cite{jswz21} and \cite{bks23} have applied HSR to facilitate solving general linear programming (LP) problems. 
Another line of research has highlighted HSR's potential in expediting the training process of contemporary neural networks \citep{qsy23, gqsw22}.
There is also a collection of research that concentrates on leveraging HSR for the advancement of solutions to geometric and graphical challenges \citep{csx05, jfl+13, egkm17}. 

\section{Preliminary} \label{sec:preliminary}

\subsection{Notations} \label{sec:preliminary:notations}

We first introduce basic notations used in this paper.
For any positive integer $n$, we use $[n]$ to denote set $\{1,2,\cdots, n\}$.  
We use $\var[]$ to denote the variance.
For two vectors $x \in \R^n$ and $y \in \R^n$, we use $\langle x, y \rangle$ to denote the inner product between $x,y$.
We use ${\bf 1}_n$ to denote a length-$n$ vector where all the entries are ones.
We use $X_{i,j}$ to denote the $i$-row, $j$-th column of $X \in \R^{m \times n}$.
We use $\|A\|_{\infty}$ to denote the $\ell_{\infty}$ norm of a matrix $A \in \R^{n \times d}$, i.e. $\|A\|_{\infty} := \max_{i \in [n], j \in [d]} |A_{i,j}|$.

\subsection{Half-Space Reporting (HSR) Data Structure} \label{sec:preliminary:hsr}

Due to the limitation of space, we provide only the core corollary of the HSR data structure here. We refer the readers to Appendix~\ref{sub:app_preli:hsr_data_structure} for more details. 

In \cite{aem92}, the author introduces a data structure named HSR to solve the half-space range reporting problem. The interface of the HSR data structure can be summarized as in Algorithm~\ref{alg:half_space_report}.
Intuitively, the HSR data structure recursively partitions the set $S$ and organizes the points in a tree data structure. Then for a given query $(a, b)$, all $k$ points of $S$ with $\sgn(\langle a, x\rangle -b)\geq 0$ are reported quickly. 
Note that the query $(a, b)$ here defines the half-space $H$ in Definition~\ref{def:HSR}.
We summarize the running time complexity of the HSR data structure as follows:
\begin{corollary}[HSR data-structure time complexity \cite{aem92}, informal version of Corollary~\ref{cor:hsr_running_time}] \label{cor:hsr_running_time:informal}

Let $\Tinit$ denote the pre-processing time to build the data structure, $\Tquery$ denote the time per query, and $\Tupdate$ time per update. Given a set of $n$ points in $\R^d$, the half-space range reporting problem can be solved with the following performances:
\begin{itemize}
    \item Part 1.
    $\Tinit(n,d)=O_d(n\log n)$, $\Tquery(n,d,k)=O(d n^{1 - 1/ \lfloor d/2\rfloor} + d k)$.
    \item Part 2.
    $\Tinit(n,d)=O(n^{\lfloor d/2\rfloor})$, $\Tquery(n,d,k)=O(d\log(n)+dk)$. 
\end{itemize}
\end{corollary}

\section{Main Results on Generation Decoding} \label{sec:main_result}

In this section, we present our key findings regarding generation decoding, $m = \Theta(1)$, for both ReLU and Softmax attention mechanisms. We reduce the time complexity from a naive $O(mn)$ to $O(mn^{4/5})$. 
Our method only introduces a negligible approximation error for Softmax attention. 

\begin{figure*}[!ht]
\centering
\includegraphics[width=1.0\textwidth]{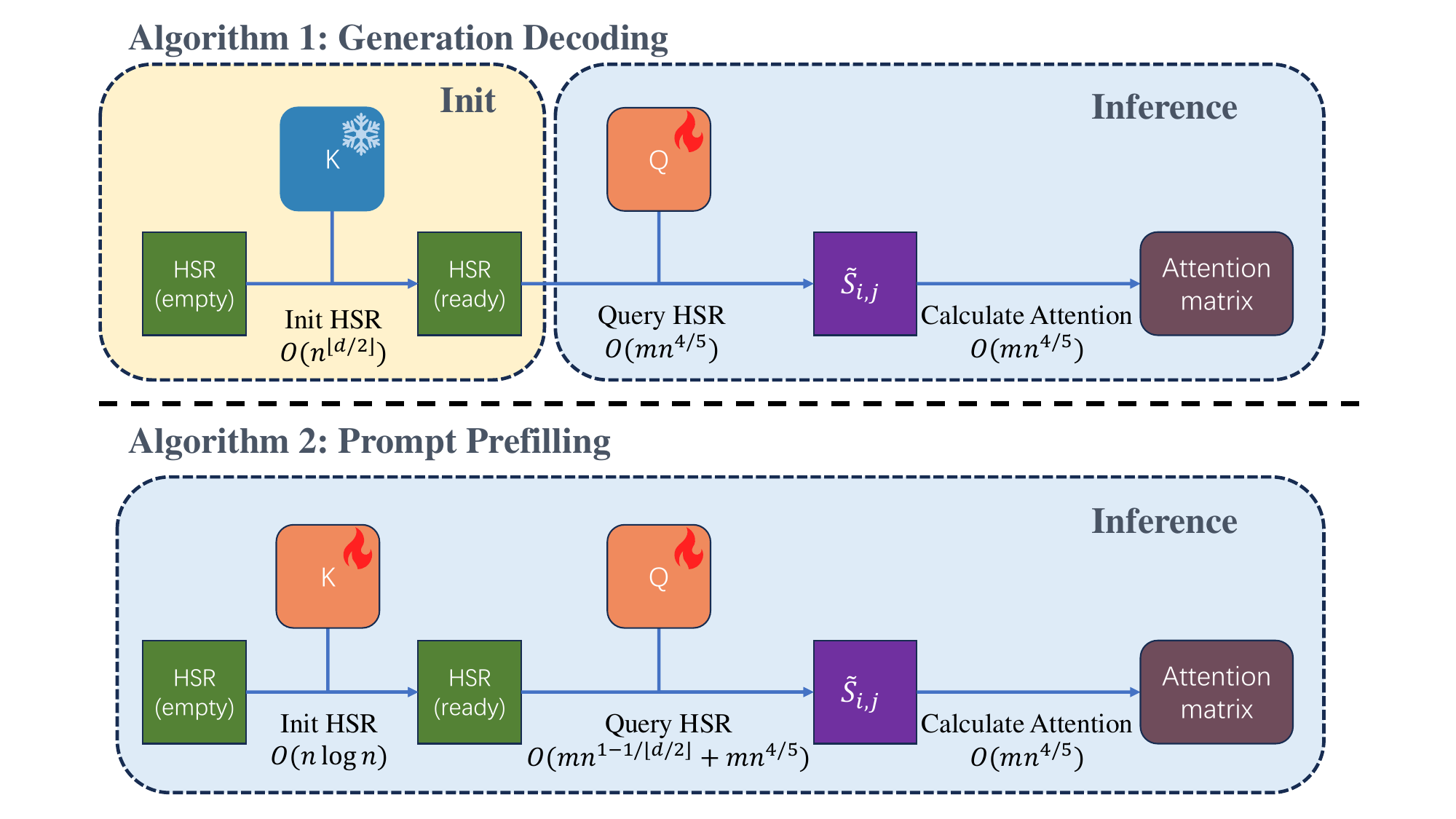}
\caption{
An outline of our principal algorithms. \textbf{Top:} Algorithm~\ref{alg:relu_attn_gen} for generation decoding is depicted, with the key matrix $K$ is fixed. During each inference step, the input query $Q$ interacts with the HSR data structure to get the activated indices set $\wt{S}_{i, j}$. Then, we can calculate the attention matrix according to $\wt{S}_{i, j}$.
\textbf{Bottom:} Algorithm~\ref{alg:calculation_general_framework} for prompt prefilling is shown, where both the key matrix $K$ and the query matrix $Q$ are variable across iterations. Consequently, the HSR data structure must first be initialized with $K$, followed by querying it using $Q$. Finally, according to the activated entries set $\wt{S}_{i, j}$ reported by the HSR data structure, the attention matrix can be calculated. For more information, please refer to Remark~\ref{rem:difference_of_alg2_and_alg3}. 
}
\label{fig:difference_of_alg2_and_alg3}
\end{figure*}

\begin{algorithm}[!ht] 
\caption{Generation decoding} 
\label{alg:relu_attn_gen}
\begin{algorithmic}[1]
\State {\bf data structure} \textsc{GenerationDecoding} 
\Comment{Lemma~\ref{lem:generation_general_framework:informal}}
\State {\bf members}
\State \hspace{4mm} \textsc{HalfSpaceReport} \textsc{hsr} \Comment{Part 2 of Corollary~\ref{cor:hsr_running_time:informal} 
}
\State \hspace{4mm} $\{K_i\}_{i \in [n]}$ \Comment{Key matrix}
\State \hspace{4mm} $V \in \R^{n \times d}$ \Comment{Value matrix}
\State \hspace{4mm} $b \in \R$ \Comment{Threshold of ReLU activation} 
\State {\bf end members}
\Procedure{Init}{$\{K_i\}_{i\in [n]}, V, n, d$}
\State $\{K_i\}_{i\in [n]},  V \gets \{K_i\}_{i\in [n]}, V$ \Comment{Store necessary matrices}
\State $b \gets \sigma_a \cdot \sqrt{0.4 \log n}$ \Comment{Init essential parameters and data structure. Lemma~\ref{lem:sparsity_analysis:informal}}
\State \textsc{hsr}.\textsc{Init}($\{K_i\}_{i\in [n]}, n, d$) \Comment{It takes $O(n^{\lfloor d/2\rfloor})$ time}
\EndProcedure
\Procedure{Inference}{$Q \in \R^{m \times d}, m$}
\State $A \gets {\bf 0}_{m \times n}$
\For{$i = 1 \to m$} \Comment{Loop for $m$ query vectors}
\State $\wt{S}_{i,\mathrm{fire}}\gets \textsc{hsr}.\textsc{Query}(Q_i, b)$ \label{line:preprocess_k:hsr_query} \Comment{It takes $O(n^{4/5})$ time}
\For{$j \in \wt{S}_{i,\mathrm{fire}}$} \Comment{Calculate the ReLU attention output according to $\wt{S}_{i,\mathrm{fire}}$}\label{line:preprocess_k:forloop_Sifire_start}
\State $A_{i, j} \gets \mathsf{ReLU}^{\alpha}( \langle Q_i, K_j \rangle /\sqrt{d} - b)$ or $A_{i, j} \gets \mathsf{Softmax}( \langle Q_i, K_j \rangle /\sqrt{d} )$ \label{line:preprocess_k:calculate_Arj}
\EndFor \label{line:preprocess_k:forloop_Sifire_end}\label{line:preprocess_k:forloop_m_end}
\EndFor
\State \Return $D^{-1} A V$ 
\EndProcedure
\State {\bf end data structure}
\end{algorithmic}
\end{algorithm}

We begin with introducing our result on ReLU attention generation decoding as follows:
\begin{theorem} [Running time of ReLU attention generation decoding, informal version of Theorem~\ref{thm:relu_gen_running_time}] \label{thm:relu_gen_running_time:informal}
Let ReLU attention be defined as Definition~\ref{def:relu_attention}. 
Assume each entry of $K$ conforms Gaussian ${\cal N}(0, \sigma_k^2)$, and each entry of $Q$ conforms Gaussian ${\cal N}(0, \sigma_q^2)$. 
Let $\delta \in (0, 1)$ denote the failure probability.
Let $\sigma_a = 4 \cdot ( 1 + d^{-1} \log(m / \delta))^{1/2} \cdot \sigma_q \sigma_k$.
Let $b = \sigma_a \cdot \sqrt{0.4 \log n}$.
Suppose we have KV Cache $K, V \in \R^{n \times d}$. We want to generate a $m$ length answer, where $n \gg m$. 
Then, our inference function in Algorithm~\ref{alg:relu_attn_gen},
with probability at least $1 - \delta$, takes $O(m n^{4/5} )$ time to generate the answer. 
\end{theorem} 

Theorem~\ref{thm:relu_gen_running_time:informal} shows that our Algorithm~\ref{alg:relu_attn_gen} accelerates the running time of ReLU attention generation decoding from naive $O(mn)$ to $O(mn^{4/5} )$, which is a significant speed up when the KV Cache is large. We provide an example for the sparsity with different $n$ in Table~\ref{tab:sparsity_example} in the Appendix.
The at least $1 - \delta$ success probability originates from the sparsity analysis of ReLU attention (Lemma~\ref{lem:sparsity_analysis:informal}), where with probability at least $1 - \delta$, we have the number of non-zero entries of each row of the attention matrix no larger than $n^{4/5}$. 

Then, we move on to presenting our result on Softmax attention generation decoding.
Our results consist of two parts: the improved running time and the approximation error analysis.
Firstly, we introduce our result about the improved running time of Softmax attention generation decoding as follows:

\begin{theorem} [Running time of Softmax attention generation decoding, informal version of Theorem~\ref{thm:Softmax_attention_generation}] \label{thm:Softmax_attention_generation:informal}
Let $Q \in \R^{m \times d}$, $K, V \in \R^{n \times d}$ and the Softmax attention $\mathsf{Attn}_s$ be defined in Definition~\ref{def:Softmax_attention}.
Let $\mathsf{NN}(r, q, K) \subseteq [n]$ 
and the Softmax attention with index set $\wh{\mathsf{Attn}}_s$ be defined as Definition~\ref{def:top_r_softmax_attention}. 
We choose the threshold $b \in \R$ in Algorithm~\ref{alg:relu_attn_gen}  such that $R = \mathsf{NN}(n^{4/5}, q, K)$. 
Then, we can show that
the Softmax attention with index set $\wh{\mathsf{Attn}}_s$ achieves
outstanding running time under the Softmax attention generation scenario:
Suppose we have KV Cache $K, V \in \R^{n \times d}$. We want to generate a $m$ length answer, where $n \gg m$. Our inference function in Algorithm~\ref{alg:relu_attn_gen} (replacing ReLU attention with Softmax attention) takes $O(m n^{4/5} )$ time to generate the answer. 
\end{theorem}

Theorem~\ref{thm:Softmax_attention_generation:informal} demonstrates that if we choose the threshold $b$ satisfying $R = \mathsf{NN}(n^{4/5}, q, K)$, we can achieve a significant running time improvement of the Softmax attention generation. 

Indeed, this method introduces an approximation error due to the exclusion of certain entries. 
Nevertheless, under mild assumptions about the distribution of the attention scores, we demonstrate that this approximation error is indeed negligible.
The proof's intuitive explanation lies in the fact that the majority of attention scores are focused on the small subset of entries that we retain. 
We organize our results as follows:
\begin{theorem}
[Error analysis of Softmax attention with index set, informal version of Theorem~\ref{thm:err_analysis_of_Softmax_attn_with_index_set}] \label{thm:err_analysis_of_Softmax_attn_with_index_set:informal}
Let $Q \in \R^{m \times d}$, $K, V \in \R^{n \times d}$ and the Softmax attention $\mathsf{Attn}_s$ be defined in Definition~\ref{def:Softmax_attention}.
Let $q \in \R^d$ denote a single row of $Q \in \R^{m \times d}$. 
Let $\gamma \in [0,1]$, $\beta_1 \ge \beta_2 \ge 0$.
Let the index set $R$ and the Softmax attention with index set $\wh{\mathsf{Attn}}_s$ be defined as Definition~\ref{def:top_r_softmax_attention}.
Let $\mathsf{NN}(r, q, K) \subseteq [n]$ denote the indices of top-$r$ entries of $q K$.
Let $R = \mathsf{NN}(n^\gamma, q, K) \subseteq [n]$, where $|R| = n^\gamma$.
Assume the query $q$ and key cache $K$ have $(\gamma, \beta_1, \beta_2)$ massive activation property (Definition~\ref{def:massive}).
Then, we have
\begin{align*}
    \| \wh{\mathsf{Attn}_s}(q, K, V) - \mathsf{Attn}_s(q, K, V) \|_\infty \le \frac{ 2 \| V  \|_{\infty}}{n^{\gamma + (\beta_1 - \beta_2)\cdot \|q\|_2 -1} } .
\end{align*} 
\end{theorem}

Theorem~\ref{thm:err_analysis_of_Softmax_attn_with_index_set:informal} presents the error of Softmax attention with index set is relatively small.
Consequently, omitting the remaining less significant entries is a justifiable compromise.
\begin{remark}
With mild assumptions on $V$, we can have more precious results from Theorem~\ref{thm:err_analysis_of_Softmax_attn_with_index_set:informal}. For example, if the entries in $V$ conform to subgaussian distribution with constant variance, 
we have $\|V\|_\infty = O(\log(n))$ with high probability. 
\end{remark}
\section{Extension on Prompt Prefilling} \label{sec:extension}

In this section, we extend our results to prompt prefilling scenarios where the number of queries and keys is proportional, i.e., $m = \Theta(n)$. 
For ReLU attention, we leverage Part 1 result of Corollary~\ref{cor:hsr_running_time:informal} to accelerate the identification of non-zero entries (activated entries). We introduce our result on ReLU attention as follows:

\begin{algorithm}[!ht] 
\caption{
Prompt Prefilling
} 
\label{alg:calculation_general_framework}
\begin{algorithmic}[1]
\State {\bf data structure} \textsc{PromptPrefilling} 
\Comment{Lemma~\ref{lem:calculation_general_framework:informal}}
\State {\bf members}
\State \hspace{4mm} \textsc{HalfSpaceReport} \textsc{hsr} \Comment{Part 1 of Corollary~\ref{cor:hsr_running_time:informal}
}
\State {\bf end members}
\State 
\Procedure{Inference}{$\{K_i\}_{i\in [n]}, \{Q_r\}_{r\in [m]}, V, n, m, d$}
\State $b \gets \sigma_a \cdot \sqrt{0.4 \log n}$. \Comment{Threshold of ReLU activation (Lemma~\ref{lem:sparsity_analysis:informal})}
\State \textsc{hsr}.\textsc{Init}($\{K_i\}_{i\in [n]}, n, d$) \Comment{It takes $O(n \log n)$ time}
\State $A \gets {\bf 0}_{m \times n}$
\For{$i = 1 \to m$} \Comment{Loop for $m$ query vectors}
\State $\wt{S}_{i,\mathrm{fire}}\gets \textsc{hsr}.\textsc{Query}(Q_i,b)$ \label{line:preprocess_q:hsr_query} \Comment{It takes $O(n^{1 - 1/\lfloor d/2\rfloor} + \wt{k}_i)$ time.}
\For{$j \in \wt{S}_{i,\mathrm{fire}}$} \label{line:preprocess_q:forloop_Sifire_start} \Comment{Calculate the ReLU attention output according to $\wt{S}_{i,\mathrm{fire}}$}
\State $A_{i, j} \gets \mathsf{ReLU}^{\alpha}( \langle Q_i, K_j \rangle /\sqrt{d} - b)$ or $A_{i, j} \gets \mathsf{Softmax}( \langle Q_i, K_j \rangle /\sqrt{d} )$ \label{line:preprocess_q:cal_Aji}
\EndFor \label{line:preprocess_q:forloop_Sifire_end}\label{line:preprocess_q:forloop_m_end}
\EndFor
\State \Return $D^{-1} A V$ 
\EndProcedure
\State {\bf end data structure}
\end{algorithmic}
\end{algorithm}

\begin{theorem} [Running time of ReLU attention prompt prefilling, informal version of Theorem~\ref{thm:relu_cal_running_time}]
\label{thm:relu_cal_running_time:informal}
Let ReLU attention be defined as Definition~\ref{def:relu_attention}. 
Assume each entry of $K$ is from Gaussian ${\cal N}(0, \sigma_k^2)$, and each entry of $Q$ is from Gaussian ${\cal N}(0, \sigma_q^2)$. 
Let $\delta \in (0, 1)$ denote the failure probability.
Let $\sigma_a = 4 \cdot ( 1 + d^{-1} \log(m / \delta))^{1/2} \cdot \sigma_q \sigma_k$.
Let $b = \sigma_a \cdot \sqrt{0.4 \log n}$.
Suppose we have $Q, K, V \in \R^{n \times d}$. 
There exist an algorithm (Algorithm~\ref{alg:calculation_general_framework}), with probability at least $1 - \delta$,
takes $O(n^{2 - 1 / \lfloor d/2\rfloor} + n^{1+4/5})$ time to compute the full ReLU attention of $Q, K, V$. 
\end{theorem}

In Theorem~\ref{thm:relu_cal_running_time:informal}, we improve the running time of ReLU attention prompt prefilling from $O(n^2)$ to $O(n^{2 - 1 / \lfloor d/2\rfloor} + n^{1+4/5})$, which is a notable uplift of the running time when $n$ is extremely large.

Then, we present our result on Softmax attention. Intuitively, we use the Part 1 result of Corollary~\ref{cor:hsr_running_time:informal} to identify those ``massive activated'' entries (top-$r$ indices) within the attention matrix of Softmax attention and calculate the Softmax attention with top-$r$ indices. We organize our results as follows:
\begin{theorem} [Running time of Softmax attention prompt prefilling, informal version of Theorem~\ref{thm:Softmax_attention_computation}] \label{thm:Softmax_attention_computation:informal}
Let $Q \in \R^{m \times d}$, $K, V \in \R^{n \times d}$ and the Softmax attention $\mathsf{Attn}_s$ be defined in Definition~\ref{def:Softmax_attention}.
Let $\mathsf{NN}(r, q, K) \subseteq [n]$ 
and the Softmax attention with index set $\wh{\mathsf{Attn}}_s$ be defined as Definition~\ref{def:top_r_softmax_attention}. 
We choose the threshold $b \in \R$ in Algorithm~\ref{alg:calculation_general_framework} such that $R = \mathsf{NN}(n^{4/5}, q, K)$. 
Then, we have
the Softmax attention with index set $\wh{\mathsf{Attn}}_s$ achieves
outstanding running time under Softmax attention prompt prefilling scenario: Suppose we have $m = \Theta(n)$. 
Algorithm~\ref{alg:calculation_general_framework} (replacing ReLU attention with Softmax attention) takes $O(n^{2 - 1 / \lfloor d/2\rfloor} + n^{1+4/5})$ time to compute the full ReLU attention of $Q, K, V$. 
\end{theorem}

Theorem~\ref{thm:Softmax_attention_computation:informal} demonstrates our $O(n^{2 - 1 / \lfloor d/2\rfloor} + n^{1+4/5})$ running time on Softmax attention prompt prefilling, which improves from naive running time $O(n^2)$. 

\section{Technical Overview} \label{sec:tech_overview}

Section~\ref{sec:tech_overview:sparsity_analysis} introduces our analysis of the sparsity in the ReLU attention mechanism.
Section~\ref{sec:tech_overview:general_attention_framework} presents our results of two general attention frameworks.
Section~\ref{sec:tech_overview:error_analysis_of_softmax_attn_with_index_set} provides our error analysis of Softmax attention with index set. We have shown that with mild assumption on the distribution of attention scores, the error of Softmax attention with index set is negligible. 

\subsection{Sparsity Analysis of ReLU Attention} \label{sec:tech_overview:sparsity_analysis}

Intuitively, the ReLU activation will deactivate some key and query pairs. 
We introduce the results of employing the concentration inequality to quantitatively analyze the number of non-zero entries.

\begin{lemma} [Sparsity analysis, informal version of Lemma~\ref{lem:sparsity_analysis}] 
\label{lem:sparsity_analysis:informal}
Let the ReLU attention be defined as Definition~\ref{def:relu_attention}. 
Let $Q \in \R^{m \times d}$ and $K, V \in \R^{n \times d}$ be defined as Definition~\ref{def:relu_attention}. 
Let $b \in \R$ denote the threshold of ReLU activation, as defined in Definition~\ref{def:relu_attention}. 
For $i \in [m]$, let $\wt{k}_i$ denote the number of non-zero entries in $i$-th row of $A \in \R^{m \times n}$.
Assume each entry of $K$ is from Gaussian ${\cal N}(0, \sigma_k^2)$, and each entry of $Q$ is from Gaussian ${\cal N}(0, \sigma_q^2)$. 
Let $\delta \in (0, 1)$ denote the failure probability.
Let $\sigma_a = 4 \cdot ( 1 + d^{-1} \log(m / \delta))^{1/2} \cdot \sigma_q \sigma_k$.
Let $b = \sigma_a \cdot \sqrt{0.4 \log n}$. 
Then, we have, 
with probability at least $1 - \delta$, for all $i \in [m]$, the number of non-zero entries of the $i$-th row $\wt{k}_i$
is at most $2 n^{4/5}$. 
\end{lemma}

In Lemma~\ref{lem:sparsity_analysis:informal}, we use $\wt{k}_i$ to denote the number of non-zero entries in $i$-th row of attention matrix $A_r \in \R^{m \times n}$. 
It indicates that if we choose $b = \sigma_a \sqrt{0.4 \log n}$, with high probability, the number of activated (non-zero) entries can be bounded by $O(n^{4/5})$. 

\subsection{General Attention Frameworks} \label{sec:tech_overview:general_attention_framework}

First, we introduce our general framework for generation decoding. 
Here, we use the Part 2 result of the HSR data structure. 
As this framework is designed for the generation decoding task, the key matrix $K$ is fixed in each inference.
Therefore, in the \textsc{Init} procedure, we initialize the HSR data structure with the key matrix $K$.
Then, in each inference, we use the same HSR data structure to answer the query from each row of the query matrix $Q$.
We provide the result of this general generation decoding framework as follows. 
\begin{lemma} [General generation decoding framework, informal version of Lemma~\ref{lem:generation_general_framework}] \label{lem:generation_general_framework:informal}
Let $Q \in \R^{m \times d}$ and $K, V \in \R^{n \times d}$ be defined as Definition~\ref{def:relu_attention}.
Assume each entry of $K$ is from Gaussian ${\cal N}(0, \sigma_k^2)$, and each entry of $Q$ is from Gaussian ${\cal N}(0, \sigma_q^2)$.
Let $\delta \in (0, 1)$ denote the failure probability.
Let $\sigma_a = 4 \cdot ( 1 + d^{-1} \log(m / \delta))^{1/2} \cdot \sigma_q \sigma_k$.
Let $b = \sigma_a \cdot \sqrt{0.4 \log n}$.
Let \textsc{hsr} data structure be defined as Part 2 in Corollary~\ref{cor:hsr_running_time:informal}.
There exists an algorithm
(Algorithm~\ref{alg:relu_attn_gen}), with at least $1 - \delta$ probability, has the following performance:
\begin{itemize}
    \item {\bf Part 1.} The \textsc{Init} procedure runs in $O(n^{\lfloor d/2\rfloor})$ time. 
    \item {\bf Part 2.} For each query, the \textsc{Inference} procedure runs in $O(m n^{4/5} )$ time. 
\end{itemize}
\end{lemma}

The general framework for prompt prefilling is quite different from the previous one.
Namely, we choose the Part 1 result of the HSR data structure.
Since in each inference, both the query matrix $Q$ and the key matrix $K$ differ from any other inference, we first initialize the HSR data structure with the key matrix $K$. 
Then, for each row of the query matrix $Q$, we query the HSR data structure to find the activated entries.
\begin{lemma} [General prompt prefilling framework, informal version of Lemma~\ref{lem:calculation_general_framework}] \label{lem:calculation_general_framework:informal}
Let $Q \in \R^{m \times d}$ and $K, V \in \R^{n \times d}$ be defined as Definition~\ref{def:relu_attention}.
Assume each entry of $K$ is from Gaussian ${\cal N}(0, \sigma_k^2)$, and each entry of $Q$ is from Gaussian ${\cal N}(0, \sigma_q^2)$. 
Let $\delta \in (0, 1)$ denote the failure probability.
Let $\sigma_a = 4 \cdot ( 1 + d^{-1} \log(m / \delta))^{1/2} \cdot \sigma_q \sigma_k$.
Let $b = \sigma_a \cdot \sqrt{0.4 \log n}$.
Let \textsc{hsr} data structure be defined as Part 1 in Corollary~\ref{cor:hsr_running_time:informal}. 
There exists an algorithm (Algorithm~\ref{alg:calculation_general_framework}), with at least $1 - \delta$ probability, computes full attention of $Q, K, V$ in $O(m n^{1 - 1/\lfloor d/2\rfloor} + m n^{4/5})$ time. 
\end{lemma}

Then, we use the following Remark to demonstrate the different intuitions behind Lemma~\ref{lem:generation_general_framework:informal} (Algorithm~\ref{alg:relu_attn_gen}) and Lemma~\ref{lem:calculation_general_framework:informal} (Algorithm~\ref{alg:calculation_general_framework}). 
\begin{remark} \label{rem:difference_of_alg2_and_alg3}
Algorithm~\ref{alg:relu_attn_gen} is tailored for generation decoding scenarios, where the key matrix $K$ remains constant throughout each inference. Consequently, our optimization efforts are directed at decreasing the time required for individual inferences, which is achieved by adopting Part 2 of Corollary~\ref{cor:hsr_running_time:informal}. In contrast, Algorithm~\ref{alg:calculation_general_framework} is intended for prompt prefilling, a context in which the key matrix $K$ varies with each inference. Thus, our objective shifts to minimizing the time complexity associated with initializing the HSR data structure, leading us to select Part 1 of Corollary~\ref{cor:hsr_running_time:informal}.
For more details, please refer to Figure~\ref{fig:difference_of_alg2_and_alg3}. 
\end{remark}

\subsection{Error Analysis of Softmax Attention with Top-\texorpdfstring{$r$}{} Indices} \label{sec:tech_overview:error_analysis_of_softmax_attn_with_index_set}

Calculating the Softmax attention on the ``massive activated'' index set will introduce an approximation error.
In the following Lemma, we analyze the quantity of this approximation error.
Here, we use $\alpha$ to denote the summation of all entries activated by $\exp(x)$ function, and we use $\ov{\alpha}$ to denote the summation of those entries which are excluded from ``massive activated'' index set.
We provide the general error bound of Softmax attention with an index set as follows:
\begin{lemma}
[General error analysis of Softmax attention with index set, informal version of Lemma~\ref{lem:Softmax_general_err_bound}] 
\label{lem:Softmax_general_err_bound:informal}
Let $Q \in \R^{m \times d}$, $K, V \in \R^{n \times d}$ and the Softmax attention $\mathsf{Attn}_s$ be defined in Definition~\ref{def:Softmax_attention}.
Let $q \in \R^d$ denote a single row of $Q \in \R^{m \times d}$. 
Let $\alpha, \ov{\alpha}$ and $\wh{\mathsf{Attn}}_s$ be defined as Definition~\ref{def:top_r_softmax_attention}. 
Then, we have
\begin{align*}
     \| \mathsf{Attn}_s(q, K, V) -  \wh{\mathsf{Attn}}_s(q, K, V) \|_{\infty} \le  \frac{2 \ov{\alpha}}{\alpha}  \cdot \| V  \|_{\infty}.
\end{align*}
\end{lemma}

Note that Lemma~\ref{lem:Softmax_general_err_bound:informal} only provides a general error analysis of Softmax attention with an index set. 
Under mild assumption on the distribution of attention scores, we show that this error is actually very small. For more details, please refer to Theorem~\ref{thm:err_analysis_of_Softmax_attn_with_index_set:informal}.

\section{Experiments} \label{sec:experiments}

In this section, we present our empirical results of evaluating three mainstream LLMs with Softmax attention with top-$r$ indices on different $r$, showing that the results of the experiments are consistent with our theoretical analysis. 

\paragraph{Datasets.} 
To estimate the approximation error of the Softmax attention with ``massive activation'' entries, we conduct experiments on the PaulGrahamEssays datasets from LLMTest-NeedleInAHaystack \citep{llm_needle_test}. Specifically, for each article in the dataset, we first input $2^{15} = 32768$ tokens to the LLMs, then generate $1024$ tokens. 

\paragraph{Metric.} We evaluate the generation quality by the classical perplexity.
Perplexity is defined as the exponentiated average negative log-likelihood of a sequence. If we have a tokenized sequence $X = (x_0, x_1, \cdots, x_N)$, then the perplexity of $X$ is:
$
    \mathrm{Perplexity}(X) = \exp({-\frac{1}{N}\sum_{i=1}^{N}\log p_\theta(x_i |x_{<i} )}),
$
where $\log p_\theta(x_i |x_{<i} )$ is the log-likelihood of the $i$-th token conditioned on the preceding tokens. 
Intuitively, it can be thought of as an evaluation of the model’s ability to predict uniformly among the set of specified tokens in a corpus. 
Importantly, the tokenization procedure has a direct impact on a model’s perplexity, which should be taken into consideration when comparing different models.

\paragraph{Models.} To demonstrate the generalization of our approximation error bound, we conducted experiments on three mainstream large models: LLaMA 3.1 8B Instruct\footnote{\scriptsize\url{https://huggingface.co/meta-llama/Meta-Llama-3.1-8B-Instruct}} \citep{llama3}, Mistral Nemo 12B Instruct\footnote{\scriptsize\url{https://huggingface.co/mistralai/Mistral-Nemo-Base-2407}} \citep{mistral_nemo}, and Phi 3.5 Mini 3.8B Instruct\footnote{\scriptsize\url{https://huggingface.co/microsoft/Phi-3.5-mini-instruct}} \citep{phi3}.

\paragraph{Results.} 
The experiments are conducted in the setting discussed in previous paragraphs.
We evaluated the performance of three mainstream LLMs using Softmax attention with top-$r$ indices. In particular, we chose $r$ from the set $\{2^2, 2^4, 2^6, 2^8, 2^{10}, 2^{12}, 2^{15}\}$. As depicted in Figure~\ref{fig:top_r_softmax_attn}, a significant increase in the perplexity (drop in performance) of LLMs is observed only when $r$ falls below $2^4$. This suggests that the ``massive activated'' tokens are predominantly found within the top-$2^4$ entries. In comparison to the total of $2^{15}$ entries, the ``massive activated'' entries constitute a relatively minor fraction. The observed results align with our theoretical analysis, confirming the approximation error of the Softmax attention with top-$r$ indices is negligible for larger values of $r$.

\begin{figure*}[!ht]
\centering
\includegraphics[width=1.0\textwidth]{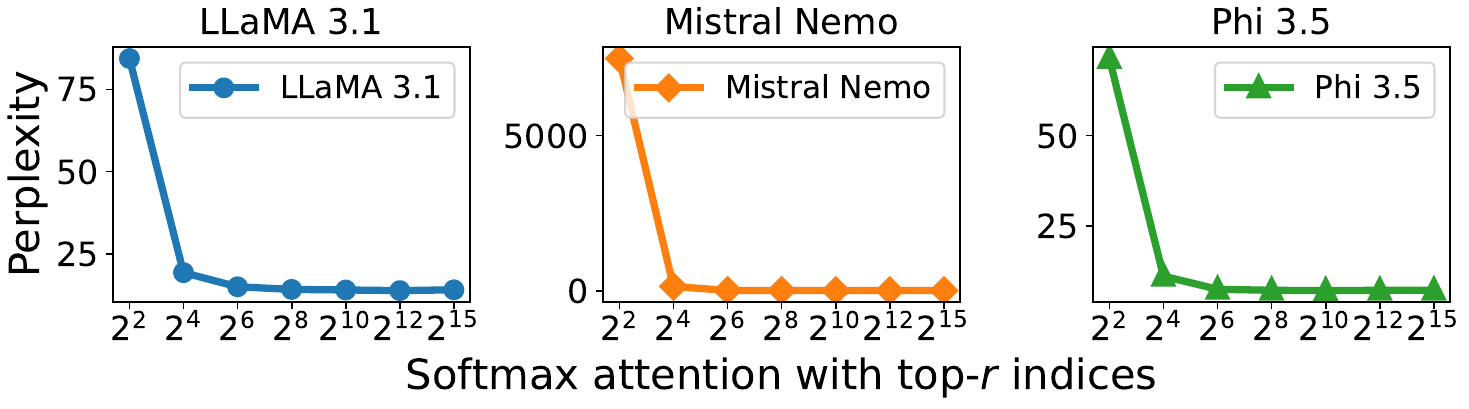}
\caption{
We evaluated the perplexity of three mainstream language models: LLaMA 3.1 8B Instruct, Mistral Nemo 12B, and Phi 3.5 Mini 3.8B Instruct,  using Softmax attention with top-$r$ indices on the PaulGrahamEssays dataset. 
The results
indicate a significant increase in perplexity only when the number of selected entries, $r$, falls below $2^4$. This observation aligns with our earlier findings that the proportion of ``massive activated'' entries is minimal compared to the total number of entries. Furthermore, the approximation error introduced by using top-$r$ indices in Softmax attention remains negligible unless $r$ becomes excessively small.
}
\label{fig:top_r_softmax_attn}
\end{figure*}

\section{Discussion and Future Work} \label{sec:discussion}

The sparsity within neural networks arises primarily from the incorporation of non-linear activation functions. 
These non-linear functions determine the mechanism or circuit of the neural networks~\citep{oen+22}. 
Gaining insight into these non-linear layers not only enhances our understanding of how neural networks work but also paves the way for optimizing training and inference. 
We believe our analysis may inspire efficient neural network architecture design. 
This work represents the initial point of this envisioned blueprint. 
We concentrate on analyzing the combinations of LLMs and fundamental non-linear activation functions, ReLU and Softmax. 
By analyzing these functions, we aim to demonstrate to the research community that a thorough examination of a model's nonlinear characteristics can significantly enhance neural networks' running-time complexity.

In real-world scenarios, various non-linear activation functions exist beyond ReLU and Softmax, such as $\mathsf{SELU}(x) = \mathrm{scale} \cdot (\max (0, x) + \min(0, \alpha \cdot (\exp(x) - 1)))$ \citep{kgmh17}, $\mathsf{CELU}(x) = \max (0, x) + \min(0, \alpha \cdot (\exp(x / \alpha) - 1))$ \citep{b17}, and $\mathsf{PRELU}(x) = \max (0, x) + \mathrm{weight} \cdot \min(0,x)$ \citep{hzrs15}. 
Analyzing these alternative functions poses multiple challenges. We will explore these additional functions in the future.

\section{Conclusion}\label{sec:conclusion}

This work investigates the exploitation of the intrinsic sparsity present in both ReLU and Softmax attention mechanisms to decrease the computational complexity of generation decoding and prompt prefilling scenarios. We employ the HSR data structure to accelerate the process of identifying non-zero or ``massive activated'' entries within ReLU and Softmax attentions. Our approach results in only a negligible approximation error for Softmax attention.

\ifdefined\isarxiv
\section*{Acknowledgement}
Research is partially supported by the National Science Foundation (NSF) Grants 2023239-DMS, CCF-2046710, and Air Force Grant FA9550-18-1-0166.
\bibliographystyle{alpha}
\bibliography{ref}
\else
\section*{Acknowledgement}
Research is partially supported by the National Science Foundation (NSF) Grants 2023239-DMS, CCF-2046710, and Air Force Grant FA9550-18-1-0166.
\bibliography{ref}

\newcommand{\etalchar}[1]{$^{#1}$}
\begin{thebibliography}{HWL{\etalchar{+}}24b}

\bibitem[AEM92]{aem92}
Pankaj~K Agarwal, David Eppstein, and Jir{\'\i} Matousek.
\newblock Dynamic half-space reporting, geometric optimization, and minimum spanning trees.
\newblock In {\em Annual Symposium on Foundations of Computer Science}, volume~33, pages 80--80. IEEE COMPUTER SOCIETY PRESS, 1992.

\bibitem[AJA{\etalchar{+}}24]{phi3}
Marah Abdin, Sam~Ade Jacobs, Ammar~Ahmad Awan, Jyoti Aneja, Ahmed Awadallah, Hany Awadalla, Nguyen Bach, Amit Bahree, Arash Bakhtiari, Harkirat Behl, et~al.
\newblock Phi-3 technical report: A highly capable language model locally on your phone.
\newblock {\em arXiv preprint arXiv:2404.14219}, 2024.

\bibitem[Ant24]{claude3.5}
Anthropic.
\newblock Claude 3.5 sonnet, 2024.

\bibitem[APB{\etalchar{+}}23]{apb+23}
Sotiris Anagnostidis, Dario Pavllo, Luca Biggio, Lorenzo Noci, Aurelien Lucchi, and Thomas Hofmann.
\newblock Dynamic context pruning for efficient and interpretable autoregressive transformers.
\newblock {\em Advances in Neural Information Processing Systems}, 36, 2023.

\bibitem[AS23]{as23}
Josh Alman and Zhao Song.
\newblock Fast attention requires bounded entries.
\newblock {\em Advances in Neural Information Processing Systems}, 36, 2023.

\bibitem[AS24a]{as24_arxiv}
Josh Alman and Zhao Song.
\newblock The fine-grained complexity of gradient computation for training large language models.
\newblock {\em arXiv preprint arXiv:2402.04497}, 2024.

\bibitem[AS24b]{as24_iclr}
Josh Alman and Zhao Song.
\newblock How to capture higher-order correlations? generalizing matrix softmax attention to kronecker computation.
\newblock In {\em The Twelfth International Conference on Learning Representations}, 2024.

\bibitem[Bar17]{b17}
Jonathan~T Barron.
\newblock Continuously differentiable exponential linear units.
\newblock {\em arXiv preprint arXiv:1704.07483}, 2017.

\bibitem[BCW{\etalchar{+}}23]{bcw+23}
Yu~Bai, Fan Chen, Huan Wang, Caiming Xiong, and Song Mei.
\newblock Transformers as statisticians: Provable in-context learning with in-context algorithm selection.
\newblock {\em Advances in neural information processing systems}, 36, 2023.

\bibitem[Ber24]{b24}
Sergei Bernstein.
\newblock On a modification of chebyshev's inequality and of the error formula of laplace.
\newblock {\em Ann. Sci. Inst. Sav. Ukraine, Sect. Math}, 1(4):38--49, 1924.

\bibitem[BKS23]{bks23}
Sayan Bhattacharya, Peter Kiss, and Thatchaphol Saranurak.
\newblock Dynamic algorithms for packing-covering lps via multiplicative weight updates.
\newblock In {\em Proceedings of the 2023 Annual ACM-SIAM Symposium on Discrete Algorithms (SODA)}, pages 1--47. SIAM, 2023.

\bibitem[CCL{\etalchar{+}}25]{ccl+25}
Yang Cao, Bo~Chen, Xiaoyu Li, Yingyu Liang, Zhizhou Sha, Zhenmei Shi, Zhao Song, and Mingda Wan.
\newblock Force matching with relativistic constraints: A physics-inspired approach to stable and efficient generative modeling.
\newblock {\em arXiv preprint arXiv:2502.08150}, 2025.

\bibitem[CDI{\etalchar{+}}21]{cdi+21}
Andy Coenen, Luke Davis, Daphne Ippolito, Emily Reif, and Ann Yuan.
\newblock Wordcraft: A human-ai collaborative editor for story writing.
\newblock {\em arXiv preprint arXiv:2107.07430}, 2021.

\bibitem[CGL{\etalchar{+}}25]{cgl+25_homo}
Bo~Chen, Chengyue Gong, Xiaoyu Li, Yingyu Liang, Zhizhou Sha, Zhenmei Shi, Zhao Song, and Mingda Wan.
\newblock High-order matching for one-step shortcut diffusion models.
\newblock {\em arXiv preprint arXiv:2502.00688}, 2025.

\bibitem[CGRS19]{cgrs19}
Rewon Child, Scott Gray, Alec Radford, and Ilya Sutskever.
\newblock Generating long sequences with sparse transformers.
\newblock {\em arXiv preprint arXiv:1904.10509}, 2019.

\bibitem[CHL{\etalchar{+}}24]{chl+24_rope}
Yifang Chen, Jiayan Huo, Xiaoyu Li, Yingyu Liang, Zhenmei Shi, and Zhao Song.
\newblock Fast gradient computation for rope attention in almost linear time.
\newblock {\em arXiv preprint arXiv:2412.17316}, 2024.

\bibitem[CLL{\etalchar{+}}24]{cll+24_rope}
Bo~Chen, Xiaoyu Li, Yingyu Liang, Jiangxuan Long, Zhenmei Shi, and Zhao Song.
\newblock Circuit complexity bounds for rope-based transformer architecture.
\newblock {\em arXiv preprint arXiv:2411.07602}, 2024.

\bibitem[CLL{\etalchar{+}}25a]{cll+25_icl}
Bo~Chen, Xiaoyu Li, Yingyu Liang, Zhenmei Shi, and Zhao Song.
\newblock Bypassing the exponential dependency: Looped transformers efficiently learn in-context by multi-step gradient descent.
\newblock In {\em International Conference on Artificial Intelligence and Statistics}, 2025.

\bibitem[CLL{\etalchar{+}}25b]{cll+25}
Bo~Chen, Xiaoyu Li, Yingyu Liang, Zhao Song, and Zhizhou Sha.
\newblock Nrflow: Towards noise-robust generative modeling via second-order flow matching.
\newblock {\em manuscript}, 2025.

\bibitem[CSX05]{csx05}
Danny~Z Chen, Michiel Smid, and Bin Xu.
\newblock Geometric algorithms for density-based data clustering.
\newblock {\em International Journal of Computational Geometry \& Applications}, 15(03):239--260, 2005.

\bibitem[CTWC24]{ctwc24}
Ruisi Cai, Yuandong Tian, Zhangyang Wang, and Beidi Chen.
\newblock Lococo: Dropping in convolutions for long context compression.
\newblock {\em arXiv preprint arXiv:2406.05317}, 2024.

\bibitem[DG24]{dg24}
Tri Dao and Albert Gu.
\newblock Transformers are ssms: Generalized models and efficient algorithms through structured state space duality.
\newblock {\em arXiv preprint arXiv:2405.21060}, 2024.

\bibitem[DYZ{\etalchar{+}}24]{dyz+24}
Harry Dong, Xinyu Yang, Zhenyu Zhang, Zhangyang Wang, Yuejie Chi, and Beidi Chen.
\newblock Get more with less: Synthesizing recurrence with kv cache compression for efficient llm inference.
\newblock {\em arXiv preprint arXiv:2402.09398}, 2024.

\bibitem[EGKM17]{egkm17}
David Eppstein, Michael~T Goodrich, Doruk Korkmaz, and Nil Mamano.
\newblock Defining equitable geographic districts in road networks via stable matching.
\newblock In {\em Proceedings of the 25th ACM SIGSPATIAL International Conference on Advances in Geographic Information Systems}, pages 1--4, 2017.

\bibitem[FA23]{fa23}
Elias Frantar and Dan Alistarh.
\newblock Sparsegpt: Massive language models can be accurately pruned in one-shot.
\newblock In {\em International Conference on Machine Learning}, pages 10323--10337. PMLR, 2023.

\bibitem[FGBM23]{fgbm23}
Hengyu Fu, Tianyu Guo, Yu~Bai, and Song Mei.
\newblock What can a single attention layer learn? a study through the random features lens.
\newblock {\em Advances in Neural Information Processing Systems}, 36, 2023.

\bibitem[GD23]{gd23}
Albert Gu and Tri Dao.
\newblock Mamba: Linear-time sequence modeling with selective state spaces.
\newblock {\em arXiv preprint arXiv:2312.00752}, 2023.

\bibitem[GQSW22]{gqsw22}
Yeqi Gao, Lianke Qin, Zhao Song, and Yitan Wang.
\newblock A sublinear adversarial training algorithm.
\newblock {\em arXiv preprint arXiv:2208.05395}, 2022.

\bibitem[GSWY23]{gswy23}
Yeqi Gao, Zhao Song, Weixin Wang, and Junze Yin.
\newblock A fast optimization view: Reformulating single layer attention in llm based on tensor and svm trick, and solving it in matrix multiplication time.
\newblock {\em arXiv preprint arXiv:2309.07418}, 2023.

\bibitem[HCL{\etalchar{+}}24]{hcl+24}
Jerry Yao-Chieh Hu, Pei-Hsuan Chang, Haozheng Luo, Hong-Yu Chen, Weijian Li, Wei-Po Wang, and Han Liu.
\newblock Outlier-efficient hopfield layers for large transformer-based models.
\newblock In {\em Forty-first International Conference on Machine Learning (ICML)}, 2024.

\bibitem[HCW{\etalchar{+}}24]{hcw+24}
Jerry Yao-Chieh Hu, Bo-Yu Chen, Dennis Wu, Feng Ruan, and Han Liu.
\newblock Nonparametric modern hopfield models.
\newblock {\em arXiv preprint arXiv:2404.03900}, 2024.

\bibitem[HDLL22]{hdll22}
Weizhe Hua, Zihang Dai, Hanxiao Liu, and Quoc Le.
\newblock Transformer quality in linear time.
\newblock In {\em International conference on machine learning}, pages 9099--9117. PMLR, 2022.

\bibitem[HJK{\etalchar{+}}24]{hjk+23}
Insu Han, Rajesh Jayaram, Amin Karbasi, Vahab Mirrokni, David Woodruff, and Amir Zandieh.
\newblock Hyperattention: Long-context attention in near-linear time.
\newblock In {\em The Twelfth International Conference on Learning Representations}, 2024.

\bibitem[HLSL24]{hlsl24}
Jerry Yao-Chieh Hu, Thomas Lin, Zhao Song, and Han Liu.
\newblock On computational limits of modern hopfield models: A fine-grained complexity analysis.
\newblock In {\em Forty-first International Conference on Machine Learning (ICML)}, 2024.

\bibitem[HWG{\etalchar{+}}24]{hwg+24}
Jerry Yao-Chieh Hu, Wei-Po Wang, Ammar Gilani, Chenyang Li, Zhao Song, and Han Liu.
\newblock Fundamental limits of prompt tuning transformers: Universality, capacity and efficiency.
\newblock {\em arXiv preprint arXiv:2411.16525}, 2024.

\bibitem[HWL{\etalchar{+}}24a]{hwl+24}
Jerry Yao-Chieh Hu, Weimin Wu, Yi-Chen Lee, Yu-Chao Huang, Minshuo Chen, and Han Liu.
\newblock On statistical rates of conditional diffusion transformers: Approximation, estimation and minimax optimality.
\newblock {\em arXiv preprint arXiv:2411.17522}, 2024.

\bibitem[HWL{\etalchar{+}}24b]{hwsl24}
Jerry Yao-Chieh Hu, Weimin Wu, Zhuoru Li, Sophia Pi, , Zhao Song, and Han Liu.
\newblock On statistical rates and provably efficient criteria of latent diffusion transformers (dits).
\newblock {\em Advances in Neural Information Processing Systems}, 38, 2024.

\bibitem[HYW{\etalchar{+}}23]{hyw+23}
Jerry Yao-Chieh Hu, Donglin Yang, Dennis Wu, Chenwei Xu, Bo-Yu Chen, and Han Liu.
\newblock On sparse modern hopfield model.
\newblock In {\em Thirty-seventh Conference on Neural Information Processing Systems (NeurIPS)}, 2023.

\bibitem[HZRS15]{hzrs15}
Kaiming He, Xiangyu Zhang, Shaoqing Ren, and Jian Sun.
\newblock Delving deep into rectifiers: Surpassing human-level performance on imagenet classification.
\newblock In {\em Proceedings of the IEEE international conference on computer vision}, pages 1026--1034, 2015.

\bibitem[JFL{\etalchar{+}}13]{jfl+13}
Wenqi Ju, Chenglin Fan, Jun Luo, Binhai Zhu, and Ovidiu Daescu.
\newblock On some geometric problems of color-spanning sets.
\newblock {\em Journal of Combinatorial Optimization}, 26:266--283, 2013.

\bibitem[JSWZ21]{jswz21}
Shunhua Jiang, Zhao Song, Omri Weinstein, and Hengjie Zhang.
\newblock A faster algorithm for solving general lps.
\newblock In {\em Proceedings of the 53rd Annual ACM SIGACT Symposium on Theory of Computing}, pages 823--832, 2021.

\bibitem[Kam24]{llm_needle_test}
Greg Kamradt.
\newblock Llmtest-needleinahaystack, 2024.

\bibitem[KLL{\etalchar{+}}25]{kll+25_var}
Yekun Ke, Xiaoyu Li, Yingyu Liang, Zhizhou Sha, Zhenmei Shi, and Zhao Song.
\newblock On computational limits and provably efficient criteria of visual autoregressive models: A fine-grained complexity analysis.
\newblock {\em arXiv preprint arXiv:2501.04377}, 2025.

\bibitem[KLS{\etalchar{+}}25]{kls+25}
Yekun Ke, Yingyu Liang, Zhenmei Shi, Zhao Song, and Chiwun Yang.
\newblock Curse of attention: A kernel-based perspective for why transformers fail to generalize on time series forecasting and beyond.
\newblock In {\em Conference on Parsimony and Learning}. PMLR, 2025.

\bibitem[KMZ23]{kmz23}
Praneeth Kacham, Vahab Mirrokni, and Peilin Zhong.
\newblock Polysketchformer: Fast transformers via sketches for polynomial kernels.
\newblock {\em arXiv preprint arXiv:2310.01655}, 2023.

\bibitem[KT19]{kjt19}
Jacob Devlin Ming-Wei~Chang Kenton and Lee~Kristina Toutanova.
\newblock Bert: Pre-training of deep bidirectional transformers for language understanding.
\newblock In {\em Proceedings of naacL-HLT}, volume~1, page~2. Minneapolis, Minnesota, 2019.

\bibitem[KUMH17]{kgmh17}
G{\"u}nter Klambauer, Thomas Unterthiner, Andreas Mayr, and Sepp Hochreiter.
\newblock Self-normalizing neural networks.
\newblock {\em Advances in neural information processing systems}, 30, 2017.

\bibitem[LBZ{\etalchar{+}}22]{lbz+22}
Zhiyuan Li, Srinadh Bhojanapalli, Manzil Zaheer, Sashank Reddi, and Sanjiv Kumar.
\newblock Robust training of neural networks using scale invariant architectures.
\newblock In {\em International Conference on Machine Learning}, pages 12656--12684. PMLR, 2022.

\bibitem[LHY{\etalchar{+}}24]{lhy+24}
Yuhong Li, Yingbing Huang, Bowen Yang, Bharat Venkitesh, Acyr Locatelli, Hanchen Ye, Tianle Cai, Patrick Lewis, and Deming Chen.
\newblock Snapkv: Llm knows what you are looking for before generation.
\newblock {\em arXiv preprint arXiv:2404.14469}, 2024.

\bibitem[LLL{\etalchar{+}}25]{lll+25_loop}
Xiaoyu Li, Yingyu Liang, Jiangxuan Long, Zhenmei Shi, Zhao Song, and Zhen Zhuang.
\newblock Neural algorithmic reasoning for hypergraphs with looped transformers.
\newblock {\em arXiv preprint arXiv:2501.10688}, 2025.

\bibitem[LLS{\etalchar{+}}24a]{lls+24_io}
Xiaoyu Li, Yingyu Liang, Zhenmei Shi, Zhao Song, and Yufa Zhou.
\newblock Fine-grained attention i/o complexity: Comprehensive analysis for backward passes, 2024.

\bibitem[LLS{\etalchar{+}}24b]{lls+24b}
Yingyu Liang, Heshan Liu, Zhenmei Shi, Zhao Song, and Junze Yin.
\newblock Conv-basis: A new paradigm for efficient attention inference and gradient computation in transformers.
\newblock {\em arXiv preprint arXiv:2405.05219}, 2024.

\bibitem[LLS{\etalchar{+}}24c]{lls+24_prune}
Yingyu Liang, Jiangxuan Long, Zhenmei Shi, Zhao Song, and Yufa Zhou.
\newblock Beyond linear approximations: A novel pruning approach for attention matrix, 2024.

\bibitem[LLS{\etalchar{+}}25]{lls+25_grok}
Chenyang Li, Yingyu Liang, Zhenmei Shi, Zhao Song, and Tianyi Zhou.
\newblock Fourier circuits in neural networks and transformers: A case study of modular arithmetic with multiple inputs.
\newblock In {\em International Conference on Artificial Intelligence and Statistics}, 2025.

\bibitem[LLSS24]{llss24_sparsegpt}
Xiaoyu Li, Yingyu Liang, Zhenmei Shi, and Zhao Song.
\newblock A tighter complexity analysis of sparsegpt.
\newblock {\em arXiv preprint arXiv:2408.12151}, 2024.

\bibitem[LM00]{lm00}
Beatrice Laurent and Pascal Massart.
\newblock Adaptive estimation of a quadratic functional by model selection.
\newblock {\em Annals of Statistics}, pages 1302--1338, 2000.

\bibitem[LSS{\etalchar{+}}24a]{lss+24_relu}
Yingyu Liang, Zhizhou Sha, Zhenmei Shi, Zhao Song, and Yufa Zhou.
\newblock Looped relu mlps may be all you need as practical programmable computers, 2024.

\bibitem[LSS{\etalchar{+}}24b]{lss+24}
Yingyu Liang, Zhizhou Sha, Zhenmei Shi, Zhao Song, and Yufa Zhou.
\newblock Multi-layer transformers gradient can be approximated in almost linear time.
\newblock {\em arXiv preprint arXiv:2408.13233}, 2024.

\bibitem[LSS{\etalchar{+}}25]{lss+25_relu}
Yingyu Liang, Zhizhou Sha, Zhenmei Shi, Zhao Song, and Yufa Zhou.
\newblock Looped relu mlps may be all you need as practical programmable computers.
\newblock In {\em International Conference on Artificial Intelligence and Statistics}, 2025.

\bibitem[LSSY24]{lssy24}
Yingyu Liang, Zhenmei Shi, Zhao Song, and Chiwun Yang.
\newblock Toward infinite-long prefix in transformer.
\newblock {\em arXiv preprint arXiv:2406.14036}, 2024.

\bibitem[LSSZ24a]{lssz24_dp}
Yingyu Liang, Zhenmei Shi, Zhao Song, and Yufa Zhou.
\newblock Differential privacy of cross-attention with provable guarantee.
\newblock {\em arXiv preprint arXiv:2407.14717}, 2024.

\bibitem[LSSZ24b]{lssz24a}
Yingyu Liang, Zhenmei Shi, Zhao Song, and Yufa Zhou.
\newblock Tensor attention training: Provably efficient learning of higher-order transformers.
\newblock {\em arXiv preprint arXiv:2405.16411}, 2024.

\bibitem[LSSZ24c]{lssz24_gm}
Yingyu Liang, Zhenmei Shi, Zhao Song, and Yufa Zhou.
\newblock Unraveling the smoothness properties of diffusion models: A gaussian mixture perspective.
\newblock {\em arXiv preprint arXiv:2405.16418}, 2024.

\bibitem[LWD{\etalchar{+}}23]{lwd+23}
Zichang Liu, Jue Wang, Tri Dao, Tianyi Zhou, Binhang Yuan, Zhao Song, Anshumali Shrivastava, Ce~Zhang, Yuandong Tian, Christopher Re, et~al.
\newblock Deja vu: Contextual sparsity for efficient llms at inference time.
\newblock In {\em International Conference on Machine Learning}, pages 22137--22176. PMLR, 2023.

\bibitem[Met24]{llama3}
Meta.
\newblock Llama 3, 2024.

\bibitem[Mis24]{mistral_nemo}
MistralAI.
\newblock Mistral nemo, 2024.

\bibitem[MVK{\etalchar{+}}24]{mvk+24}
Jean Mercat, Igor Vasiljevic, Sedrick Keh, Kushal Arora, Achal Dave, Adrien Gaidon, and Thomas Kollar.
\newblock Linearizing large language models.
\newblock {\em arXiv preprint arXiv:2405.06640}, 2024.

\bibitem[OEN{\etalchar{+}}22]{oen+22}
Catherine Olsson, Nelson Elhage, Neel Nanda, Nicholas Joseph, Nova DasSarma, Tom Henighan, Ben Mann, Amanda Askell, Yuntao Bai, Anna Chen, et~al.
\newblock In-context learning and induction heads.
\newblock {\em arXiv preprint arXiv:2209.11895}, 2022.

\bibitem[Ope23]{gpt4turbo}
OpenAI.
\newblock Gpt-4 turbo, 2023.

\bibitem[PMN{\etalchar{+}}23]{pmn+23}
Michael Poli, Stefano Massaroli, Eric Nguyen, Daniel~Y Fu, Tri Dao, Stephen Baccus, Yoshua Bengio, Stefano Ermon, and Christopher R{\'e}.
\newblock Hyena hierarchy: Towards larger convolutional language models.
\newblock In {\em International Conference on Machine Learning}, pages 28043--28078. PMLR, 2023.

\bibitem[PQFS23]{pqj+23}
Bowen Peng, Jeffrey Quesnelle, Honglu Fan, and Enrico Shippole.
\newblock Yarn: Efficient context window extension of large language models.
\newblock {\em arXiv preprint arXiv:2309.00071}, 2023.

\bibitem[QSY23]{qsy23}
Lianke Qin, Zhao Song, and Yuanyuan Yang.
\newblock Efficient sgd neural network training via sublinear activated neuron identification.
\newblock {\em arXiv preprint arXiv:2307.06565}, 2023.

\bibitem[QSZZ23]{qszz23}
Lianke Qin, Zhao Song, Lichen Zhang, and Danyang Zhuo.
\newblock An online and unified algorithm for projection matrix vector multiplication with application to empirical risk minimization.
\newblock In {\em International Conference on Artificial Intelligence and Statistics (AISTATS)}, pages 101--156. PMLR, 2023.

\bibitem[RWC{\etalchar{+}}19]{rwc+19}
Alec Radford, Jeffrey Wu, Rewon Child, David Luan, Dario Amodei, Ilya Sutskever, et~al.
\newblock Language models are unsupervised multitask learners.
\newblock {\em OpenAI blog}, 1(8):9, 2019.

\bibitem[SCKL24]{sckl24}
Mingjie Sun, Xinlei Chen, J~Zico Kolter, and Zhuang Liu.
\newblock Massive activations in large language models.
\newblock {\em arXiv preprint arXiv:2402.17762}, 2024.

\bibitem[SGT{\etalchar{+}}23]{sgt+23}
Kai Shen, Junliang Guo, Xu~Tan, Siliang Tang, Rui Wang, and Jiang Bian.
\newblock A study on relu and softmax in transformer.
\newblock {\em arXiv preprint arXiv:2302.06461}, 2023.

\bibitem[SMN{\etalchar{+}}24]{smn+24}
Zhenmei Shi, Yifei Ming, Xuan-Phi Nguyen, Yingyu Liang, and Shafiq Joty.
\newblock Discovering the gems in early layers: Accelerating long-context llms with 1000x input token reduction.
\newblock {\em arXiv preprint arXiv:2409.17422}, 2024.

\bibitem[SSZ{\etalchar{+}}24a]{ssz+24_dit}
Xuan Shen, Zhao Song, Yufa Zhou, Bo~Chen, Yanyu Li, Yifan Gong, Kai Zhang, Hao Tan, Jason Kuen, Henghui Ding, et~al.
\newblock Lazydit: Lazy learning for the acceleration of diffusion transformers.
\newblock {\em arXiv preprint arXiv:2412.12444}, 2024.

\bibitem[SSZ{\etalchar{+}}24b]{ssz+24_pruning}
Xuan Shen, Zhao Song, Yufa Zhou, Bo~Chen, Jing Liu, Ruiyi Zhang, Ryan~A Rossi, Hao Tan, Tong Yu, Xiang Chen, et~al.
\newblock Numerical pruning for efficient autoregressive models.
\newblock {\em arXiv preprint arXiv:2412.12441}, 2024.

\bibitem[SWXL24]{swxl24}
Zhenmei Shi, Junyi Wei, Zhuoyan Xu, and Yingyu Liang.
\newblock Why larger language models do in-context learning differently?
\newblock {\em arXiv preprint arXiv:2405.19592}, 2024.

\bibitem[SYYZ23]{syyz23_dp}
Zhao Song, Xin Yang, Yuanyuan Yang, and Lichen Zhang.
\newblock Sketching meets differential privacy: fast algorithm for dynamic kronecker projection maintenance.
\newblock In {\em International Conference on Machine Learning (ICML)}, pages 32418--32462. PMLR, 2023.

\bibitem[SYZ21]{syz21}
Zhao Song, Shuo Yang, and Ruizhe Zhang.
\newblock Does preprocessing help training over-parameterized neural networks?
\newblock {\em Advances in Neural Information Processing Systems}, 34:22890--22904, 2021.

\bibitem[SYZ23]{syz23}
Zhao Song, Mingquan Ye, and Lichen Zhang.
\newblock Streaming semidefinite programs: $ {O} (\sqrt n)$ passes, small space and fast runtime.
\newblock {\em arXiv preprint arXiv:2309.05135}, 2023.

\bibitem[SZZ24]{szz24}
Zhao Song, Lichen Zhang, and Ruizhe Zhang.
\newblock Training multi-layer over-parametrized neural network in subquadratic time.
\newblock In {\em Innovations in Theoretical Computer Science (ITCS)}, pages 93:1--93:15, 2024.

\bibitem[TDFH{\etalchar{+}}22]{tdh+22}
Romal Thoppilan, Daniel De~Freitas, Jamie Hall, Noam Shazeer, Apoorv Kulshreshtha, Heng-Tze Cheng, Alicia Jin, Taylor Bos, Leslie Baker, Yu~Du, et~al.
\newblock Lamda: Language models for dialog applications.
\newblock {\em arXiv preprint arXiv:2201.08239}, 2022.

\bibitem[TZZ{\etalchar{+}}24]{tzz+24}
Jiaming Tang, Yilong Zhao, Kan Zhu, Guangxuan Xiao, Baris Kasikci, and Song Han.
\newblock Quest: Query-aware sparsity for efficient long-context llm inference.
\newblock {\em arXiv preprint arXiv:2406.10774}, 2024.

\bibitem[VSP{\etalchar{+}}17]{vsp+17}
Ashish Vaswani, Noam Shazeer, Niki Parmar, Jakob Uszkoreit, Llion Jones, Aidan~N Gomez, {\L}ukasz Kaiser, and Illia Polosukhin.
\newblock Attention is all you need.
\newblock {\em Advances in neural information processing systems}, 30, 2017.

\bibitem[WCZ{\etalchar{+}}23]{wcz+23}
Yilin Wang, Zeyuan Chen, Liangjun Zhong, Zheng Ding, Zhizhou Sha, and Zhuowen Tu.
\newblock Dolfin: Diffusion layout transformers without autoencoder.
\newblock {\em arXiv preprint arXiv:2310.16305}, 2023.

\bibitem[WHHL24]{whhl24}
Dennis Wu, Jerry Yao-Chieh Hu, Teng-Yun Hsiao, and Han Liu.
\newblock Uniform memory retrieval with larger capacity for modern hopfield models.
\newblock In {\em Forty-first International Conference on Machine Learning (ICML)}, 2024.

\bibitem[WHL{\etalchar{+}}24]{whl+24}
Dennis Wu, Jerry Yao-Chieh Hu, Weijian Li, Bo-Yu Chen, and Han Liu.
\newblock {ST}anhop: Sparse tandem hopfield model for memory-enhanced time series prediction.
\newblock In {\em The Twelfth International Conference on Learning Representations (ICLR)}, 2024.

\bibitem[WLGK23]{wlgk23}
Mitchell Wortsman, Jaehoon Lee, Justin Gilmer, and Simon Kornblith.
\newblock Replacing softmax with relu in vision transformers.
\newblock {\em arXiv preprint arXiv:2309.08586}, 2023.

\bibitem[WMS{\etalchar{+}}24]{wms+24}
Jiayu Wang, Yifei Ming, Zhenmei Shi, Vibhav Vineet, Xin Wang, Yixuan Li, and Neel Joshi.
\newblock Is a picture worth a thousand words? delving into spatial reasoning for vision language models.
\newblock {\em Advances in Neural Information Processing Systems}, 36, 2024.

\bibitem[WSD{\etalchar{+}}23]{wsd+23}
Zirui Wang, Zhizhou Sha, Zheng Ding, Yilin Wang, and Zhuowen Tu.
\newblock Tokencompose: Grounding diffusion with token-level supervision.
\newblock {\em arXiv preprint arXiv:2312.03626}, 2023.

\bibitem[WSH{\etalchar{+}}24]{wsh+24}
Weimin Wu, Maojiang Su, Jerry Yao-Chieh Hu, Zhao Song, and Han Liu.
\newblock Transformers are deep optimizers: Provable in-context learning for deep model training.
\newblock {\em arXiv preprint arXiv:2411.16549}, 2024.

\bibitem[WTB{\etalchar{+}}22]{wtb+22}
Jason Wei, Yi~Tay, Rishi Bommasani, Colin Raffel, Barret Zoph, Sebastian Borgeaud, Dani Yogatama, Maarten Bosma, Denny Zhou, Donald Metzler, et~al.
\newblock Emergent abilities of large language models.
\newblock {\em arXiv preprint arXiv:2206.07682}, 2022.

\bibitem[WXZ{\etalchar{+}}24]{wxz+24}
Yilin Wang, Haiyang Xu, Xiang Zhang, Zeyuan Chen, Zhizhou Sha, Zirui Wang, and Zhuowen Tu.
\newblock Omnicontrolnet: Dual-stage integration for conditional image generation.
\newblock In {\em Proceedings of the IEEE/CVF Conference on Computer Vision and Pattern Recognition}, pages 7436--7448, 2024.

\bibitem[XHH{\etalchar{+}}24]{xhh+24}
Chenwei Xu, Yu-Chao Huang, Jerry Yao-Chieh Hu, Weijian Li, Ammar Gilani, Hsi-Sheng Goan, and Han Liu.
\newblock Bishop: Bi-directional cellular learning for tabular data with generalized sparse modern hopfield model.
\newblock In {\em Forty-first International Conference on Machine Learning (ICML)}, 2024.

\bibitem[XSL24]{xsl24}
Zhuoyan Xu, Zhenmei Shi, and Yingyu Liang.
\newblock Do large language models have compositional ability? an investigation into limitations and scalability.
\newblock In {\em ICLR 2024 Workshop on Mathematical and Empirical Understanding of Foundation Models}, 2024.

\bibitem[ZBKR24]{zbkr24}
Michael Zhang, Kush Bhatia, Hermann Kumbong, and Christopher R{\'e}.
\newblock The hedgehog \& the porcupine: Expressive linear attentions with softmax mimicry.
\newblock {\em arXiv preprint arXiv:2402.04347}, 2024.

\bibitem[Zha22]{z22}
Lichen Zhang.
\newblock {\em Speeding up optimizations via data structures: Faster search, sample and maintenance}.
\newblock PhD thesis, Master’s thesis, Carnegie Mellon University, 2022.

\bibitem[ZHMK24]{zhmk24}
Amir Zandieh, Insu Han, Vahab Mirrokni, and Amin Karbasi.
\newblock Subgen: Token generation in sublinear time and memory.
\newblock {\em arXiv preprint arXiv:2402.06082}, 2024.

\bibitem[ZLD{\etalchar{+}}24]{zld+24}
Tianyi Zhang, Faisal Ladhak, Esin Durmus, Percy Liang, Kathleen McKeown, and Tatsunori~B Hashimoto.
\newblock Benchmarking large language models for news summarization.
\newblock {\em Transactions of the Association for Computational Linguistics}, 12:39--57, 2024.

\bibitem[ZSZ{\etalchar{+}}23]{zsz+23}
Zhenyu Zhang, Ying Sheng, Tianyi Zhou, Tianlong Chen, Lianmin Zheng, Ruisi Cai, Zhao Song, Yuandong Tian, Christopher R{\'e}, Clark Barrett, et~al.
\newblock H2o: Heavy-hitter oracle for efficient generative inference of large language models.
\newblock {\em Advances in Neural Information Processing Systems}, 36, 2023.

\end{thebibliography}

\fi

\newpage
\onecolumn
\appendix

\begin{center}
    \textbf{\LARGE Appendix}
\end{center}


{\bf Roadmap.} 
In Section~\ref{sec:app:limitation}, we discuss some limitations of our work. 
In Section~\ref{sec:app_preli}, we introduce more fundamental lemmas and facts. 
In Section~\ref{sec:app_relu_attn_calculation}, we extend the analysis to ReLU attention calculation, demonstrating improved performance over standard attention computation under specific conditions.
In Section~\ref{sec:app_relu_attn_generation}, we first introduce and analyze the time complexity of ReLU attention generation using half-space reporting (HSR) data structures. 
In Section~\ref{sec:app_sparsity}, we analyze the sparsity of ReLU attention matrices. 
In Section~\ref{sec:app:running_time_of_softmax_attention}, we introduce our results on reducing the running time of Softmax attention. 
In Section~\ref{sec:app_error_Softmax}, we analyze error bounds for Softmax attention with index sets, balancing efficiency and accuracy.

\begin{table}[h!]
\centering
\begin{tabular}{c | c c } 
 \hline
 \toprule
 Sequence & Activated  & Sparsity \\  
 length & entries & ratio \\  
 \midrule
 $1$k & 251 & 0.75 \\
$2$k & 437 & 0.78 \\
$4$k & 761 & 0.81 \\
$8$k & 1325 & 0.83 \\
$16$k & 2308 & 0.86 \\
$32$k & 4019 & 0.87 \\
$64$k & 6997 & 0.89 \\
$128$k & 12183 & 0.90 \\
$256$k & 21212 & 0.92 \\
$512$k & 36933 & 0.93 \\
$1024$k & 64304 & 0.94 \\
 \bottomrule
 \hline
\end{tabular}
\caption{
An illustration of the sparsity level attained by our algorithm across varying sequence lengths, $n$. Our approach activates merely $n^{4/5}$ entries per inference, resulting in a computational savings of up to $90\%$ when $n = 1024$k.
}
\label{tab:sparsity_example}
\end{table}

\section{Limitations} \label{sec:app:limitation}

Our work is fundamentally theoretical in nature. We concentrate on the theoretical analysis of our proposed algorithms (Algorithms \ref{alg:relu_attn_gen} and \ref{alg:calculation_general_framework}). Our study does not include an implementation of the suggested algorithms, a limitation that arises from the absence of an existing implementation for the original HSR data structure, initially proposed in \cite{aem92}. We are confident that our theoretical insights will inspire the development of future algorithmic designs.

\section{Full Background and Definition}\label{sec:app_preli}
In this section, we display more fundamental concepts. 
In Section~\ref{sec:preliminary:Softmax_attn}, we introduce a modified version of Softmax attention that operates on a specific subset of indices. It defines the top-$r$ nearest neighbors Softmax attention, which focuses on the most relevant entries in the attention matrix. In Section~\ref{sec:preliminary:Softmax_attn_extension:massive_activation}, we describe the massive activation property for attention mechanisms. 
In Section~\ref{sub:app_preli:probability_tools}, we introduce several important probability properties and bounds. In Section~\ref{sub:app_preli:hsr_data_structure}, we detail the time complexity and performance of half-space reporting (HSR) data structures.

\subsection{Softmax Attention with Index Set} \label{sec:preliminary:Softmax_attn}

Recall that we have already provided the definition of ReLU attention in Definition~\ref{def:relu_attention}. Here, we present the key concepts of Softmax attention. 
For Softmax attention, since we only calculate the ``massive activated'' entries to get our approximated results, we introduce the formal definition: 
\begin{definition}
[Input with index set]
\label{def:input_index}
Let $K \in \R^{n \times d}$ and $V \in \R^{n \times d}$ be defined in Definition~\ref{def:Softmax_attention}. 
Let $R \subseteq [n]$ be an index set of size $|R| = r \in [n]$.  
Let $\ov{R}:= [n] \setminus R$ be the complementary set, where $|\ov{R}| = n-r$.
We define 
\begin{align*}
    \wh{K} := K_{R} \in \R^{r \times d} \quad \wh{V} := V_{R} \in \R^{r \times d} \quad \ov{K}:= K_{\ov{R}} \in \R^{(n-r) \times d} \quad \ov{V}:= V_{\ov{R}} \in \R^{(n-r) \times d}
\end{align*}
as the submatrix of $K$ and $V$, i.e., whose row index is in $R$ or $\ov{R}$, respectively.  
\end{definition}

We consider calculating the Softmax attention on the ``massive activation'' index set, where we define the ``massive activation'' index set as the top-$r$ indices, where the formal definition is as follows:
\begin{definition}[
Top-$r$ indices Softmax attention 
]\label{def:top_r_softmax_attention}
Let $q \in \R^d$, $K, V \in \R^{n \times d}$  be defined in Definition~\ref{def:Softmax_attention}.
Let $\mathsf{NN}(r, q, K) \subseteq [n]$ denote the indices of top-$r$ entries of $q K$, where $|\mathsf{NN}(r, q, K)| = r$. 
Let $\wh{K}, \wh{V} \in \R^{r \times d}$ and $\ov{K}, \ov{V} \in \R^{(n-r) \times d}$ be defined in Definition~\ref{def:input_index}.
We define the top-$r$ nearest neighbors (NN) Softmax attention computation $\wh{\mathsf{Attn}}_s(q, K, V) \in \R^{d}$ as follows:
\begin{align*}
    \wh{\mathsf{Attn}}_s(q, K, V) := \mathsf{Softmax}(q \wh{K}^\top) \wh{V} = \wh{\alpha}^{-1} \wh{u} \wh{V} \in \R^{d}
\end{align*}
where
\begin{align*}
    \wh{u} := \exp( q \wh{K}^\top) \in \R^r \quad \text{and} \quad \wh{\alpha} := \langle \wh{u}, {\bf 1}_{r} \rangle \in \R.
\end{align*}
Furthermore, we define $\ov{u} := \exp( q \ov{K}^\top) \in \R^{n-r}$, $\ov{\alpha} := \langle \ov{u}, {\bf 1}_{n-r} \rangle \in \R$, and $u := \exp( q K^\top) \in \R^{n+1} $, $\alpha := \langle u, {\bf 1}_{n+1} \rangle \in \R$. 
\end{definition}

In Definition~\ref{def:top_r_softmax_attention}, we view the ``massive activated'' entries as the top-$r$ entries. Therefore, we only calculate the Softmax attention based on $\wh{K}, \wh{V} \in \R^{r \times d}$, instead of $K, V \in \R^{n \times d}$.

\subsection{Massive Activation} \label{sec:preliminary:Softmax_attn_extension:massive_activation}
Now, we introduce our observations on the properties of the attention scores (the inner products of query vectors and key vectors). 
This further facilitates the error analysis of the top-$r$ indices Softmax attention. To begin with, we provide the definition of the massive activation property as follows: 
\begin{definition}[Massive activation property]\label{def:massive}
Let $\gamma \in [0,1]$, $\beta_1 \ge \beta_2 \ge 0$.
 Let $\mathsf{NN}(r, q, K) \subseteq [n]$ denote the indices of top-$r$ entries of $q K^\top$.
We define $(\gamma, \beta_1, \beta_2)$ massive activation for a query $q \in \R^d$ and key cache $K \in \R^{n \times d}$, if the following conditions hold:
\begin{itemize}[leftmargin=*]
    \item The top-$n^\gamma$ entries are massive, i.e., 
    $\frac{1}{n^\gamma \cdot \|q\|_2} \sum_{i\in \mathsf{NN}(n^\gamma, q, K)} \langle q, K_i \rangle  \ge  \beta_1 \log(n) $. 
    \item The remaining terms are upper bounded, i.e, 
    $\forall i \in [n] \setminus \mathsf{NN}(n^\gamma, q, K)$, $ \frac{1}{\|q\|_2} \langle q, K_i \rangle \le \beta_2 \log(n)$. 
\end{itemize}

\end{definition}

An intuitive understanding of Definition~\ref{def:massive} is that the summation of ``massive activated'' entries dominates the summation of all entries, and the entries we ignored only contribute little to the final summation. Therefore, it is reasonable for us to omit those non ``massive activated'' entries. 

\begin{remark}
There are many distributions satisfying the property in Definition~\ref{def:massive}, such as
(1) $K$ drawing from any subexponential distribution, e.g., multivariate Laplace distributions, (2) $K$ drawing from any mixture of Gaussian distribution with $n^{1-\gamma}$ Gaussian clusters. 
\end{remark}

\subsection{Probability Tools}\label{sub:app_preli:probability_tools}
We state several fundamental properties and bounds for some common distributions.
\begin{fact} [Weighted summation of Gaussian] \label{fact:Gaussian_weighted_sum}
If the following conditions hold:
\begin{itemize}
    \item Let $x \in \R^d$ be a fixed vector and $y \in \R^d$ be a random vector. 
    \item For $i \in [d]$, let $x_i$ denote the $i$-th entry of $x$. 
    \item Suppose for $i \in [d]$, $y_i \sim \N(0, \sigma^2)$. 
\end{itemize}

Then the inner product of $x$ and $y$, $\langle x, y \rangle$ conforms Gaussian distribution $\N(0, \| x\|_2^2 \sigma^2)$. 
Namely, we have $\langle x, y \rangle \sim \N(0, \| x\|_2^2 \sigma^2)$. 
\end{fact}

\begin{fact} [Independence between $\langle x, y_i \rangle$ and $\langle x, y_j \rangle$] \label{fact:yi_yj_independent}
If the following conditions hold:
\begin{itemize}
    \item Let $x \in \R^d$ be a fixed vector. 
    \item Let $y_1, y_2, \cdots y_n \in \R^d$ be $n$ random vectors. 
    \item For any $i, j \in [n], i \neq j$, $y_i$ and $y_j$ are independent.
\end{itemize}

Then, for any $i, j \in [n], i \neq j$, $\langle x, y_i \rangle$ and $\langle x, y_j \rangle$ are independent. 
\end{fact}

We provide tail bounds for chi-square and Gaussian distributed random variables:

\begin{lemma}[Chi-square tail bound, Lemma 1 in \cite{lm00} 
]\label{lem:chi_square_bound}
    Let $X \sim \mathcal{X}_k^2$ be a chi-squared distributed random variable with $k$ degrees of freedom. Each one has zero means and $\sigma^2$ variance. 
    
    Then, it holds that
    \begin{align*}
        \Pr[X - k\sigma^2 \geq (2\sqrt{k t} + 2t) \sigma^2]
        \leq & ~ \exp{(-t)} \\
        \Pr[k\sigma^2 - X \geq 2\sqrt{k t}\sigma^2]
        \leq & ~ \exp{(-t)}
    \end{align*}
\end{lemma}

\begin{fact} [Gaussian tail bound] \label{fact:gaussian_tail_bound}
Suppose we have a random variable $x \sim \N(\mu, \sigma)$.

Then, for $t \in \R$, we have
\begin{align*}
    \Pr [x \geq \mu + t] \leq \exp(-\frac{t^2}{2 \sigma^2})
\end{align*}
\end{fact}

\begin{proof}

    We can show 
\begin{align}\label{eq:markov_lambda}
    \Pr [x \geq \mu + t] =&~ \Pr [x - \mu \geq t] \notag\\
    =&~ \Pr [e^{x - \mu} \geq e^t] \notag \\
    =&~ \inf_{\lambda \geq 0} \Pr [ e^{\lambda(x - \mu)} \geq e^{\lambda t}] \notag \\
    \leq &~ \inf_{\lambda \geq 0}\frac{\E [e^{\lambda(x - \mu)}]}{e^{\lambda t}}
\end{align}
where the first step, the second step follows from basic algebra, the third step follows from that the inequality holds for any $\lambda >0$, and the fourth step follows from Markov's inequality.

Then we consider the numerator, and we use $y = x - \mu$ to simplify the calculation, we have
\begin{align}\label{eq:expectation_of_e_lambda_y}
    \E [e^{\lambda y}] =&~ \int_\R e^{\lambda y} \frac{e^{-y^2/2\sigma^2}}{\sqrt{2\pi }\sigma} \d y \notag \\
    =&~ \int_\R \frac{e^{-(y-\lambda / \sigma^2)^2 \cdot \frac{1}{2\sigma^2} e^{\lambda^2 \sigma^2 /2}}}{\sqrt{2 \pi \sigma}} \d y \notag \\
    =&~ e^{\frac{\lambda^2 \sigma^2}{2}} \int_\R \frac{e^{-(y-\lambda / \sigma^2) \cdot \frac{1}{2\sigma^2}}}{\sqrt{2 \pi} \sigma}\d y \notag \\
    =&~ e^{\frac{\lambda^2 \sigma^2}{2}}
\end{align}
where the first step follows from the definition of the moment generating function, the second and the third steps follow from basic algebra, and the fourth step follows from the property of the probability density function.

Then we have 
\begin{align*}
    \Pr [x \geq \mu + t] \leq &~ \inf_{\lambda \geq 0} \exp(\frac{\lambda^2 - \sigma^2}{2} - \lambda t) \\
    \leq &~ \exp(- \frac{t^2}{2\sigma^2})
\end{align*}
where the first step follows from Eq.~\eqref{eq:markov_lambda} and Eq.\eqref{eq:expectation_of_e_lambda_y}, the second step follows from the calculation of infimum.
\end{proof}

The Bernstein's inequality for bounding sums of independent random variables is:
\begin{lemma}[Bernstein inequality \cite{b24}]\label{lem:bernstein}
Assume $Z_1, \cdots, Z_n$ are $n$ i.i.d. random variables. $\forall i \in [n]$, $\E[Z_i]=0$ and $|Z_i| \leq M$ almost surely. Let $Z = \sum_{i=1}^n Z_i$. Then,
\begin{align*}
\Pr \left[ Z > t \right] \leq \exp \left( - \frac{ t^2/2 }{ \sum_{j=1}^n \E[Z_j^2]  + M t /3 } \right), \forall t > 0.
\end{align*}
\end{lemma}

\subsection{Half-Space Reporting (HSR) Data Structures} \label{sub:app_preli:hsr_data_structure}

\begin{algorithm}[!ht]
\caption{Half Space Report Data Structure}
\label{alg:half_space_report}
\begin{algorithmic}[1]
    \algrenewcommand\algorithmicprocedure{\textbf{data structure}}
    \Procedure{HalfSpaceReport}{}
        \State \hspace{4mm} \textsc{Init}($S, n, d$) \Comment{Initialize the data structure with a set $S$ of $n$ points in $\R^d$}
        \State \hspace{4mm} \textsc{Query}($a,b$)\Comment{$a,b\in \R^d$. Output the set $\{x\in S: \sgn(\langle a, x\rangle -b)\geq 0\}$}
    \EndProcedure
\end{algorithmic}
\end{algorithm}

We restate the result from \cite{aem92} for solving the half-space range reporting problem. 
The half-space range reporting problem is a fundamental problem in computational geometry and can be formally defined as follows:
\begin{definition}[Half-space range reporting \citep{aem92, syz21}]\label{def:HSR}
Given a set $S$ of $n$ points in $\R^d$ with initialization, we have an operation \textsc{Query}$(H)$: given a half-space $H \subset \R^d$, output all of the points in $S$ that contain in $H$, i.e., $S\cap H$.
\end{definition}

The time complexity of the HSR data structure is:
\begin{theorem}[Agarwal, Eppstein and Matousek~\cite{aem92}]\label{thm:aem92}
Let $d$ be a fixed constant. Let $t$ be a parameter between $n$ and $n^{\lfloor d/2\rfloor}$. 
There is a dynamic data structure for half-space reporting that uses $O_{d,\epsilon}(t^{1+\epsilon})$ space and pre-processing time,  $O_{d,\epsilon}(\frac{n}{t^{1/\lfloor d/2\rfloor}}\log n+ k)$ time per query where $k$ is the output size and $\epsilon>0$ is any fixed constant, and $O_{d,\epsilon}(t^{1+\epsilon}/n)$ amortized update time. 
\end{theorem}

As a direct corollary, we have
\begin{corollary}[HSR data-structure time complexity \cite{aem92}, formal version of Corollary~\ref{cor:hsr_running_time:informal}] \label{cor:hsr_running_time}

Let $\Tinit$ denote the pre-processing time to build the data structure, $\Tquery$ denote the time per query, and $\Tupdate$ time per update.
Given a set of $n$ points in $\R^d$, the half-space range reporting problem can be solved with the following performances:
\begin{itemize}
    \item Part 1.
    $\Tinit(n,d)=O_d(n\log n)$, $\Tquery(n,d,k)=O(d n^{1 - 1/ \lfloor d/2\rfloor} + d k)$.
    \item Part 2.
    $\Tinit(n,d)=O(n^{\lfloor d/2\rfloor})$, $\Tquery(n,d,k)=O(d\log(n)+dk)$. 
\end{itemize}
\end{corollary}

\section{ReLU Attention Prompt Prefilling}\label{sec:app_relu_attn_calculation}
In this section, we focus on optimizing the standard ReLU attention calculation. By leveraging a HSR data structure and assuming sparsity, the time complexity can be reduced to $O( n^{1+4/5} d )$. 

\begin{lemma} [General full attention computation framework, formal version of Lemma~\ref{lem:calculation_general_framework:informal}] \label{lem:calculation_general_framework}
If the following conditions hold:
\begin{itemize}
    \item Let $Q \in \R^{m \times d}$ and $K, V \in \R^{n \times d}$ be defined as Definition~\ref{def:relu_attention}.
    \item Assume each entry of $K$ is from Gaussian ${\cal N}(0, \sigma_k^2)$, and each entry of $Q$ is from Gaussian ${\cal N}(0, \sigma_q^2)$. 
    \item Let $\delta \in (0, 1)$ denote the failure probability.
    \item Let $\sigma_a = 4 \cdot ( 1 + d^{-1} \log(m / \delta))^{1/2} \cdot \sigma_q \sigma_k$.
    \item Let $b = \sigma_a \cdot \sqrt{0.4 \log n}$.
    \item Let \textsc{hsr} data structure be defined as Part 1 in Corollary~\ref{cor:hsr_running_time}. 
\end{itemize}
There exists an algorithm (Algorithm~\ref{alg:calculation_general_framework}), with at least $1 - \delta$ probability, computes full attention of $Q, K, V$ in $O(m n^{1 - 1/\lfloor d/2\rfloor} + m n^{4/5})$ time. 
\end{lemma}

\begin{proof}
For $i \in [m]$, let $\wt{k}_i := |\wt{S}_{i, \mathrm{fire}}|$ denote the number of non-zero entries in $i$-th row of $A \in \R^{m \times n}$.  

The running time for \textsc{Inference} procedure can be written as
\begin{align*}
    \Tinit(n, d) 
    + \sum_{i=1}^m \Tquery(n, d, \wt{k}_i) 
    + O(d \sum_{i=1}^m \wt{k}_i)
    + O(d \sum_{i=1}^m \wt{k}_i)
\end{align*}

The first term $\Tinit(n, d)$ corresponds to the initialization of the \textsc{hsr} data structure. 
Since we use the Part 1 result from Corollary~\ref{cor:hsr_running_time}, the running time for initialization is $\Tinit(n,d) = O_d(n \log n)$. 

The second term $\sum_{i=1}^m \Tquery(n, d, \wt{k}_i)$ comes from the HSR query operation (Line~\ref{line:preprocess_q:hsr_query}). 
Since we use Part 1 result from Corollary~\ref{cor:hsr_running_time}, we have
\begin{align*}
    \sum_{i=1}^m \Tquery(n, d, \wt{k}_i) 
    = & ~ O(m n^{1 - 1/\lfloor d/2\rfloor}d + d\sum_{i=1}^m \wt{k}_i) \\
    = & ~ O(m n^{1 - 1/\lfloor d/2\rfloor}d + m n^{4/5}d)
\end{align*}
where the first step follows from $\Tquery(n, d, \wt{k}_i) = O(d n^{1 - \lfloor d/2\rfloor} + d \wt{k}_i)$ (Part 1 of Corollary~\ref{cor:hsr_running_time}), 
the second step follows from with high probability $\wt{k}_i$ at most $n^{4/5}$ (Lemma~\ref{lem:sparsity_analysis}). 

The third term $O(\sum_{i=1}^m \wt{k}_i)$ corresponds to calculating $A_{j, i}$ (Line~\ref{line:preprocess_q:cal_Aji}). 
By Lemma~\ref{lem:sparsity_analysis}, we have the third term is $O(m n^{4/5})$. 

The fourth term $O(\sum_{i=1}^m \wt{k}_i)$ corresponds to calculating $D^{-1} A V$. 
Since for $i$-th row of $A$, there are $\wt{k}_i$ non-zero entries. Therefore, it takes $O(\sum_{i=1}^m \wt{k}_i)$ time for calculating $D^{-1} A$. 
Therefore, it takes $O(d \sum_{i=1}^m \wt{k}_i)$ time to calculate $D^{-1} A V$. 
By Lemma~\ref{lem:sparsity_analysis}, with high probability, $\wt{k}_i$ is at most $n^{4/5}$. Therefore, we have the third term as $O(m n^{4/5} d)$. 

To sum up, 
the overall running time 
is $O(m n^{1 - 1/\lfloor d/2\rfloor}d + m n^{4/5}d)$. 
\end{proof}

We can now derive a more specific result for the full ReLU attention computation:

\begin{theorem} [Running time of full ReLU attention computation, formal version of Lemma~\ref{thm:relu_cal_running_time:informal}] \label{thm:relu_cal_running_time}
If the following conditions hold:
\begin{itemize}
    \item Let ReLU attention be defined as Definition~\ref{def:relu_attention}. 
    \item Assume each entry of $K$ is from Gaussian ${\cal N}(0, \sigma_k^2)$, and each entry of $Q$ is from Gaussian ${\cal N}(0, \sigma_q^2)$. 
    \item Let $\delta \in (0, 1)$ denote the failure probability. 
    \item Let $\sigma_a = 4 \cdot ( 1 + d^{-1} \log(m / \delta))^{1/2} \cdot \sigma_q \sigma_k$.
    \item Let $b = \sigma_a \cdot \sqrt{0.4 \log n}$.
    \item Suppose we have $Q, K, V \in \R^{n \times d}$. 
\end{itemize}

There exists an algorithm (Algorithm~\ref{alg:calculation_general_framework}), with probability at least $1 - \delta$,
takes $O(n^{2 - 1 / \lfloor d/2\rfloor} d + n^{1+4/5} d)$ time to compute the full ReLU attention of $Q, K, V$. 
\end{theorem}

\begin{proof}
By Lemma~\ref{lem:calculation_general_framework}, we have that the \textsc{FullAttentionComputation} data structure (Algorithm~\ref{alg:calculation_general_framework}) can run \textsc{Inference} to calculate the ReLU attention, in $O(m^{1 - \lfloor d/2\rfloor} n d + m n^{4/5} d)$ time. 

By our assumption, we have $Q \in \R^{n \times d}$.
For each calculation, we only need to call \textsc{FullAttentionComputation}.\textsc{Inference}$(K, Q, V, n, n, d)$ for once.

Then, we have the ReLU attention calculation run in $O(n^{1+4/5} d)$ time. 
\end{proof}

\section{ReLU Attention Generation Decoding}\label{sec:app_relu_attn_generation}

In this section, we present a theoretical analysis of the time complexity of ReLU attention generation using an HSR data structure.

\begin{lemma} [General attention generation framework, formal version of Lemma~\ref{lem:generation_general_framework:informal}] \label{lem:generation_general_framework}
If the following conditions hold:
\begin{itemize}
    \item Let $Q \in \R^{m \times d}$ and $K, V \in \R^{n \times d}$ be defined as Definition~\ref{def:relu_attention}.
    \item Assume each entry of $K$ is from Gaussian ${\cal N}(0, \sigma_k^2)$, and each entry of $Q$ is from Gaussian ${\cal N}(0, \sigma_q^2)$.
    \item Let $\delta \in (0, 1)$ denote the failure probability.
    \item Let $\sigma_a = 4 \cdot ( 1 + d^{-1} \log(m / \delta))^{1/2} \cdot \sigma_q \sigma_k$.
    \item Let $b = \sigma_a \cdot \sqrt{0.4 \log n}$.
    \item Let \textsc{hsr} data structure be defined as Part 2 in Corollary~\ref{cor:hsr_running_time}.
\end{itemize}

Then, there exists an algorithm
(Algorithm~\ref{alg:relu_attn_gen}), with at least $1 - \delta$ probability, has the following performance:
\begin{itemize}
    \item {\bf Part 1.} The \textsc{Init} procedure runs in $O(n^{\lfloor d/2\rfloor})$ time. 
    \item {\bf Part 2.} For each query, the \textsc{Inference} procedure runs in $O(m n^{4/5} d)$ time. 
\end{itemize}
\end{lemma}

\begin{proof}
{\bf Proof of Part 1.}

The \textsc{Init} procedure only runs the initialization of the HSR data structure. 
Since we use Part 2 result from Corollary~\ref{cor:hsr_running_time}, the running time of \textsc{Init} procedure is $\Tinit(n,d)=O(n^{\lfloor d/2\rfloor})$. 

{\bf Proof of Part 2.}

For $i \in [m]$, let $\wt{k}_i := |\wt{S}_{i, \mathrm{fire}}|$ denote the number of non-zero entries in $i$-th row of $A \in \R^{m \times n}$.  

The running time for \textsc{Inference} procedure can be written as
\begin{align*}
    \sum_{i=1}^m \Tquery(n, d, \wt{k}_i) 
    + O(d \sum_{i=1}^m \wt{k}_i) 
    + O(d \sum_{i=1}^m \wt{k}_i)
\end{align*}

The first term $\sum_{i=1}^m \Tquery(n, d, \wt{k}_i)$ corresponds to the HSR query operation (Line~\ref{line:preprocess_k:hsr_query}). Since we use the Part 2 result from Corollary~\ref{cor:hsr_running_time}, we have
\begin{align*}
    \sum_{i=1}^m \Tquery(n, d, \wt{k}_i) 
    = & ~ O(m d \log n + d \sum_{i=1}^m \wt{k}_i) \\
    = & ~ O(m d \log n + m n^{4/5} d) \\
    = & ~ O(m n^{4/5}d )
\end{align*}
where the first step follows from $\Tquery(n, d, k) = O(d \log n + d k)$ in Part 2 of Corollary~\ref{cor:hsr_running_time}, the second step follows from with high probability, $\wt{k}_i$ is at most $n^{4/5}$ (Lemma~\ref{lem:sparsity_analysis}), the third step follows from $\log n < n^{4/5}$. 

The second term $O(d \sum_{i=1}^m \wt{k}_i)$ corresponds to calculating $A_{i, j}$ (Line~\ref{line:preprocess_k:calculate_Arj}). 
There are $m$ iterations, and in each iteration, it calculates $\wt{k}_i$ entries of $A$. 
Then, the second term is $O(d \sum_{i=1}^m \wt{k}_i)$. 
By Lemma~\ref{lem:sparsity_analysis}, with high probability, $\wt{k}_i$ is at most $n^{4/5}$. Therefore, we have the second term as $O(m n^{4/5}d )$. 

Similar to the proof of Lemma~\ref{lem:calculation_general_framework} this term is $O(m n^{4/5} d)$.

To sum up, 
we have the overall running time for \textsc{Inference} procedure is $O(m n ^{4/5} d)$. 
\end{proof}

We now derive a comprehensive sparsity analysis for the ReLU attention mechanism:

\begin{theorem} [Running time of full ReLU attention generation, formal version of Theorem~\ref{thm:relu_gen_running_time:informal}] \label{thm:relu_gen_running_time}
If the following conditions hold:
\begin{itemize}
    \item Let ReLU attention be defined as Definition~\ref{def:relu_attention}. 
    \item Assume each entry of $K$ is from Gaussian ${\cal N}(0, \sigma_k^2)$, and each entry of $Q$ is from Gaussian ${\cal N}(0, \sigma_q^2)$. 
    \item Let $\delta \in (0, 1)$ denote the failure probability.
    \item Let $\sigma_a = 4 \cdot ( 1 + d^{-1} \log(m / \delta))^{1/2} \cdot \sigma_q \sigma_k$.
    \item Let $b = \sigma_a \cdot \sqrt{0.4 \log n}$.
    \item Suppose we have KV Cache $K, V \in \R^{n \times d}$. We want to generate a $m$ length answer, where $n \gg m$. 
\end{itemize}

There exists an algorithm (Algorithm~\ref{alg:relu_attn_gen}),
with at least $1 - \delta$ probability,
takes $O(m  n^{4/5} d)$ time to generate the answer. 
\end{theorem}

\begin{proof}
We make use of the \textsc{AttentionGeneration}  data structure (Algorithm~\ref{alg:relu_attn_gen}) in Lemma~\ref{lem:generation_general_framework}. 

The generation process is an auto-regressive procedure, we define the following notations for better understanding. 
For $i \in [m]$, let $q_i, k_i \in \R^d$ denote the query vector of the $i$-th iteration, respectively. 
Note that $q_i$ need to attend on both $K \in \R^{n \times d}$ and $\{k_1, k_2, \cdots, k_{i-1}\}$. 

For calculating the attention between $q_i$ and $K \in \R^{n \times d}$, we just need to call \textsc{AttentionGeneration} .\textsc{Inference}$(q_i, 1)$ for once.
Therefore, the running time for this part is $O(n^{4/5} d)$ time. 

For calculating the attention between $q_i$ and $\{k_1, k_2, \cdots, k_{i-1}, k_i\}$, it takes $O(i \cdot d)$ time. 

Therefore, for a single query $q_i$, the running time for getting the attention matrix $A \in \R^{1 \times (n+i)}$ is $(n^{4/5} + i) \cdot d$.
Since there are only $n^{4/5} + i$ non-zero entries in $A$, it takes $n^{4/5} + i$ time to calculate $D^{-1} A$.
Then, it takes $(n^{4/5} + i) \cdot d$ time to calculate $D^{-1} A V$. 
Since $i \leq m$, the total running time for calculating attention for a single query $q_i$ is $O((n^{4/5} + m) \cdot d)$. 

There are $m$ queries in total. The running time for $m$ queries is $O(m n^{4/5} d + m^2 d)$.

Since we have $n \gg m$, 
the overall running time for the generation is $O(m  n^{4/5}  d)$. 
\end{proof}

\section{Sparsity Analysis}\label{sec:app_sparsity}

To begin our analysis, we first examine the application of Bernstein's inequality to the matrix $K$:

\begin{lemma} [Bernstein on $K$] \label{lem:bernstein_on_k}
If the following conditions hold:
\begin{itemize}
    \item Let the ReLU attention be defined as Definition~\ref{def:relu_attention}. 
    \item Let $Q \in \R^{m \times d}$ and $K, V \in \R^{n \times d}$ be defined as Definition~\ref{def:relu_attention}. 
    \item Let $b \in \R$ denote the threshold of ReLU activation, as defined in Definition~\ref{def:relu_attention}. 
    \item For $i \in [m]$, let $\wt{k}_i$ denote the number of non-zero entries in $i$-th row of $A \in \R^{m \times n}$. 
    \item Assume each entry of $K$ is from Gaussian ${\cal N}(0, \sigma_k^2)$  
    \item Let $x \in \R^d$ denote a single row of $Q \in \R^{m \times d}$.
    \item Let $\sigma_a = \| x \|_2 \sigma_k / \sqrt{d}$.   
\end{itemize}

Then, we can show that,
with probability at least $1 - \exp(- \Omega(n \cdot \exp(- \frac{b^2}{2 \sigma_a^2})))$, the number of non-zero entries $\wt{k}_i$ is at most $2n \cdot \exp(- \frac{b^2}{2 \sigma_a^2})$. 
Namely, we have
\begin{align*}
    \Pr[\wt{k}_i \leq 2n \cdot \exp(- \frac{b^2}{2 \sigma_a^2})] \geq 1 - \exp(- \Omega(n \cdot \exp(- \frac{b^2}{2 \sigma_a^2})))
\end{align*}
\end{lemma}

\begin{proof}
For simplicity, for $i \in [n], j \in [d]$, we use $K_{i, j} \in \R$ to denote the $(i, j)$-th entry of $K \in \R^{n \times d}$. 

Let $r_i \in \{0, 1\}$ be the indicator function of $\langle x, K_{i, *} \rangle$. 
Then, we have $\wt{k}_i = \sum_{j=1}^n r_j$. 

Since $r_i$ is an indicator function, then we have
\begin{align*}
    |r_i| \leq 1. 
\end{align*}

By assumption, we have $K_{i, j} \sim \N(0, \sigma_k^2)$. 

Let $\sigma_a = \| x \|_2 \cdot \sigma_k / \sqrt{d}$. 

By the property of Gaussian distribution (Fact~\ref{fact:Gaussian_weighted_sum}), we have $\langle x, K_{i, *} \rangle \sim \N(0, d \cdot \sigma_a^2)$ and $\langle x, K_{i, *} \rangle / \sqrt{d} \sim \N(0, \sigma_a^2)$. 

For any $i, j \in [n]$, by Fact~\ref{fact:yi_yj_independent}, we have $\langle x, K_{i, *} \rangle$ and $\langle x, K_{j, *} \rangle$ are independent, which implies $r_i$ and $r_j$ are independent. 

By the tail bound of Gaussian distribution (Fact~\ref{fact:gaussian_tail_bound}), we have
\begin{align*}
    \Pr[r_i = 1] 
    = & ~ \Pr [\langle x, K_{i, *} \rangle / \sqrt{d} \geq b] \\
    \leq & ~ \exp(- \frac{b^2}{2 \sigma_a^2}), 
\end{align*}
which implies 
\begin{align} \label{eq:ri_expectation}
    \E [r_i] \leq \exp(- \frac{b^2}{2 \sigma_a^2}), 
\end{align}
and 
\begin{align*}
    \E [r_i^2] \leq \exp(- \frac{b^2}{2 \sigma_a^2}), 
\end{align*}
which implies
\begin{align*}
    \sum_{i=1}^n \E [r_i^2] \leq n \cdot \exp(- \frac{b^2}{2 \sigma_a^2}). 
\end{align*}

Since we have $\wt{k}_i = \sum_{j=1}^n r_j$, 
by Eq.~\eqref{eq:ri_expectation}, we have
\begin{align*}
    E[\wt{k}_i] \leq n \cdot \exp(- \frac{b^2}{2 \sigma_a^2}). 
\end{align*}

Let $k_0 := n \cdot \exp(- \frac{b^2}{2 \sigma_a^2})$. 
By the Bernstein inequality (Lemma~\ref{lem:bernstein}), we have
\begin{align} \label{eq:r_bernstein}
    \Pr[\wt{k}_i \geq k_0 + t] \leq \exp (-\frac{t^2 / 2}{k_0 + t / 3})
\end{align} 

We choose $t = k_0$, then we have
\begin{align*}
    \Pr [\wt{k}_i \geq 2 k_0] \leq \exp(- 3 k_0 / 8)
\end{align*}

Then, we reach our conclusion: with probability at least $1 - \exp(- \Omega(n \cdot \exp(- \frac{b^2}{2 \sigma_a^2})))$, the number of non-zero entries in each row of the attention matrix $A$ is bounded by $\wt{k}_i \leq 2n \cdot \exp(- \frac{b^2}{2 \sigma_a^2})$. 

\end{proof}

We turn our attention to bounding $\| x \|_2$: 

\begin{lemma} [$\| x \|_2$ bound] \label{lem:x_2nrom_bound}
If the following conditions hold:
\begin{itemize}
    \item Let $Q \in \R^{m \times d}$ be defined as Definition~\ref{def:relu_attention}.
    \item Let $x \in \R^d$ denote a single row of $Q \in \R^{m \times d}$.
    \item Assume each entry of $Q$ is from $\N(0, \sigma_q^2)$. 
\end{itemize}

Then, we can show that, for $t \geq 0$ with probability $1 - \exp(-t)$, $\| x \|_2$ is at most $\sqrt{3}\cdot(d+t)^{1/2}\cdot \sigma_q$. 
Namely, we have
\begin{align*}
    \Pr[\| x \|_2 \leq \sqrt{3}\cdot(d+t)^{1/2}\cdot \sigma_q] \geq 1 - \exp(-t). 
\end{align*}
\end{lemma}

\begin{proof}
For simplicity, we use $x_i \in \R$ to denote the $i$-th entry of $x$. 

By the assumption, we have $x_i \sim \N(0, \sigma_q^2)$. 

Since $\| x \|_2^2 = \sum_{i = 1}^d x_i^2$, by Chi-square tail bound (Lemma~\ref{lem:chi_square_bound}), we have
\begin{align*}
    \Pr [\| x \|_2^2 - d \sigma_q^2 \geq (2 \sqrt{dt} + 2t) \sigma_q^2] \leq \exp(-t), 
\end{align*}
which implies
\begin{align} \label{eq:x_2norm_square_bound}
    \Pr [\| x \|_2^2 \geq (2 \sqrt{dt} + 2t +d) \sigma_q^2] \leq \exp(-t). 
\end{align}

Since we have $2 \sqrt{dt} \leq d + t$, Eq.~\eqref{eq:x_2norm_square_bound} implies
\begin{align*}
    \Pr [\| x \|_2^2 \geq 3(d+t) \sigma_q^2] \leq \exp(-t), 
\end{align*}
which is equivalent to 
\begin{align*}
    \Pr [\| x \|_2 \geq \sqrt{3} \cdot (d+t)^{1/2} \cdot \sigma_q] \leq \exp(-t). 
\end{align*}
\end{proof}

We can now present our formal sparsity analysis, which builds upon the previous lemmas:

\begin{lemma} [Sparsity analysis, formal version of Lemma~\ref{lem:sparsity_analysis:informal}] \label{lem:sparsity_analysis}
If the following conditions hold:
\begin{itemize}
    \item Let the ReLU attention be defined as Definition~\ref{def:relu_attention}. 
    \item Let $Q \in \R^{m \times d}$ and $K, V \in \R^{n \times d}$ be defined as Definition~\ref{def:relu_attention}. 
    \item Let $b \in \R$ denote the threshold of ReLU activation, as defined in Definition~\ref{def:relu_attention}. 
    \item For $i \in [m]$, let $\wt{k}_i$ denote the number of non-zero entries in $i$-th row of $A \in \R^{m \times n}$. 
    \item Assume each entry of $K$ is from Gaussian ${\cal N}(0, \sigma_k^2)$, and each entry of $K$ is from Gaussian ${\cal N}(0, \sigma_q^2)$. 
    \item Let $\delta \in (0, 1)$ denote the failure probability. 
    \item Let $\sigma_a = 4 \cdot ( 1 + d^{-1} \log(m / \delta))^{1/2} \cdot \sigma_q \sigma_k$. 
    \item Let $b = \sigma_a \cdot \sqrt{0.4 \log n}$. 
\end{itemize}

Then, we can show that, 
with probability at least $1 - \delta$, for all $i \in [m]$, the number of non-zero entries of the $i$-th row $\wt{k}_i$
is at most $2 n^{4/5}$. 
\end{lemma}

\begin{proof}
This proof follows from applying union bound on Lemma~\ref{lem:bernstein_on_k} and Lemma~\ref{lem:x_2nrom_bound}. 

By Lemma~\ref{lem:x_2nrom_bound}, we have
\begin{align} \label{eq:x_2norm_bound_with_t}
    \Pr[\| x \|_2 \leq \sqrt{3}\cdot(d+t)^{1/2}\cdot \sigma_q] \geq 1 - \exp(-t). 
\end{align}

We choose $t = d + \log (m / \delta)$. Then, Eq.~\eqref{eq:x_2norm_bound_with_t} implies
\begin{align} \label{eq:x_2norm_bound}
    \Pr[\| x \|_2 \leq 4 \cdot (d + \log(m / \delta))^{1/2} \cdot \sigma_q] \geq 1 -  \exp(- ( d + \log (m / \delta) ) ). 
\end{align}

Let $\sigma_a = \| x \|_2 \cdot \sigma_k / \sqrt{d}$. 
By Eq.\eqref{eq:x_2norm_bound}, 
we have $\sigma_a = 4 \cdot ( 1 + d^{-1} \log(m / \delta))^{1/2} \cdot \sigma_q \sigma_k$. 

By Lemma~\ref{lem:bernstein_on_k}, we have
\begin{align} \label{eq:r_leq_2n}
    \Pr[\wt{k}_i \leq 2n \cdot \exp(- \frac{b^2}{2 \sigma_a^2})] \geq 1 - \exp(- \Omega(n \cdot \exp(- \frac{b^2}{2 \sigma_a^2}))). 
\end{align}

Let $b = \sigma_a \cdot \sqrt{0.4 \log n}$. 
Then, Eq.~\eqref{eq:r_leq_2n} implies
\begin{align} \label{eq:r_leq_n_4over5}
    \Pr [\wt{k}_i \leq 2 n^{4 / 5}] \geq 1 - \exp (-O(n^{4 / 5}))
\end{align}

Since we have $n \gg d$, this implies
\begin{align} \label{eq:exp_n_4over5_leq_exp_d}
    \exp(-O(n^{4/5})) \leq \exp(-d)
\end{align}

Taking union bound over Eq.~\eqref{eq:x_2norm_bound} and Eq.~\eqref{eq:r_leq_n_4over5}, we have
\begin{align} \label{eq:single_r_bound}
    \Pr [\wt{k}_i \leq 2 n^{4 / 5}] 
    \geq & ~ 1 - (\exp (-O(n^{4 / 5}) +  \exp(- ( d + \log(m/\delta) ) )) \notag \\
    = &~ 1 - (\exp (-O(n^{4 / 5}) +  (\delta / m) \cdot \exp(- d ) )) \notag \\
    \geq & ~ 1 - \delta/m. 
\end{align}
where the first step follows from the union bound, the second step follows from basic algebra, the third step follows from Eq.~\eqref{eq:exp_n_4over5_leq_exp_d}. 

Since $x \in \R$ represents a single row of $Q \in \R^{m \times d}$, we already proved that for each fixed row of $A$, the $\wt{k}_i$ is at most $2n^{4/5}$ with probability $1-\delta/m$.

Taking the union bound over $m$ rows in $A$, then we can show that with probability $1-\delta$, for all rows of $A$, that row's $\wt{k}_i$ is at most $2n^{4/5}$.

\end{proof}

\section{Running Time of Softmax Attention} \label{sec:app:running_time_of_softmax_attention}

In this section, we provide our results on reducing the running time of Softmax attention. 
We begin with introducing our result on Softmax attention generation.

\begin{theorem} [Running time of Softmax attention generation, formal version of Theorem~\ref{thm:Softmax_attention_generation:informal}] \label{thm:Softmax_attention_generation}
Let $Q \in \R^{m \times d}$, $K, V \in \R^{n \times d}$ and the Softmax attention $\mathsf{Attn}_s$ be defined in Definition~\ref{def:Softmax_attention}.
Let $\mathsf{NN}(r, q, K) \subseteq [n]$ 
and the Softmax attention with index set $\wh{\mathsf{Attn}}_s$ be defined as Definition~\ref{def:top_r_softmax_attention}. 
We choose the threshold $b \in \R$ in Algorithm~\ref{alg:relu_attn_gen} such that $R = \mathsf{NN}(n^{4/5}, q, K)$. 
Then, we can show that the Softmax attention with index set $\wh{\mathsf{Attn}}_s$ achieves
outstanding running time under the Softmax attention generation scenario:
Suppose we have KV Cache $K, V \in \R^{n \times d}$. We want to generate a $m$ length answer, where $m = \Theta(1)$. Algorithm~\ref{alg:relu_attn_gen} (replacing ReLU attention with Softmax attention) takes $O(m n^{4/5} )$ time to generate the answer. 
\end{theorem}

\begin{proof}
The Softmax attention generation scenario can be proved by substituting the ReLU attention $\mathsf{Attn}_r$ (Definition~\ref{def:relu_attention}) with Softmax attention with index set $\wh{\mathsf{Attn}}_s$ (Definition~\ref{def:top_r_softmax_attention}) in Algorithm~\ref{alg:relu_attn_gen} and Theorem~\ref{thm:relu_gen_running_time:informal}. 
\end{proof}

Then, we move on to our result on Softmax full attention computation.

\begin{theorem} [Running time of Softmax full attention computation, formal version of Theorem~\ref{thm:Softmax_attention_computation:informal}] \label{thm:Softmax_attention_computation}
Let $Q \in \R^{m \times d}$, $K, V \in \R^{n \times d}$ and the Softmax attention $\mathsf{Attn}_s$ be defined in Definition~\ref{def:Softmax_attention}.
Let $\mathsf{NN}(r, q, K) \subseteq [n]$ 
and the Softmax attention with index set $\wh{\mathsf{Attn}}_s$ be defined as Definition~\ref{def:top_r_softmax_attention}. 
We choose the threshold $b \in \R$ in Algorithm~\ref{alg:calculation_general_framework} such that $R = \mathsf{NN}(n^{4/5}, q, K)$. 
Then, we can show that the Softmax attention with index set $\wh{\mathsf{Attn}}_s$ achieves
outstanding running time under full Softmax attention computation scenario: Suppose we have $m = \Theta(n)$.
Algorithm~\ref{alg:calculation_general_framework} (replacing ReLU attention with Softmax attention) takes $O(n^{2 - 1 / \lfloor d/2\rfloor}d  + n^{1+4/5} d)$ time to calculate the attention output. 
\end{theorem}

\begin{proof}
The Softmax full attention computation scenario can be proved by substituting the ReLU attention $\mathsf{Attn}_r$ (Definition~\ref{def:relu_attention}) with Softmax attention with index set $\wh{\mathsf{Attn}}_s$ (Definition~\ref{def:top_r_softmax_attention}) in Algorithm~\ref{alg:calculation_general_framework} and Theorem~\ref{thm:relu_cal_running_time:informal}. 
\end{proof}

\section{Error Analysis of Softmax Attention}\label{sec:app_error_Softmax}
In this section, we provide an error analysis of the Softmax attention mechanism, deriving error bounds for the general case and a specific case with the massive activation property.

The following lemmas establish error bounds for Softmax attention when using index sets, formalizing the approximation error in attention computation.
\begin{lemma}
[
General error analysis of Softmax attention with index set, formal version of Lemma~\ref{lem:Softmax_general_err_bound:informal}
] 
\label{lem:Softmax_general_err_bound}
If the following conditions hold:
\begin{itemize}
    \item Let $Q \in \R^{m \times d}$, $K, V \in \R^{n \times d}$ and the Softmax attention $\mathsf{Attn}_s$ be defined in Definition~\ref{def:Softmax_attention}.
    \item Let $q \in \R^d$ denote a single row of $Q \in \R^{m \times d}$. 
    \item 
    Let $\alpha, \ov{\alpha}$ and $\wh{\mathsf{Attn}}_s$ be defined as Definition~\ref{def:top_r_softmax_attention}. 
\end{itemize}

Then we have
\begin{align*}
    \| \mathsf{Attn}_s(q, K, V) -  \wh{\mathsf{Attn}}_s(q, K, V) \|_{\infty} \le  & ~ \frac{2 \ov{\alpha}}{\alpha}  \cdot \| V  \|_{\infty}.
\end{align*}
\end{lemma}

\begin{proof}
Recall that $\ov{R} = [n] \setminus R$ and $\wh{K} = K_{R} \in \R^{r \times d}$ and $\wh{V} = V_{R} \in \R^{r \times d}$ and $\ov{K} = K_{\ov{R}} \in \R^{(n-r) \times d}$ and $\ov{V} = V_{\ov{R}} \in \R^{(n-r) \times d}$ as defined in Definition~\ref{def:input_index}. 
Also, we have $\wh{u} = \exp( q \wh{K}^\top) \in \R^r $ and $\wh{\alpha} = \langle \wh{u}, {\bf 1}_{r} \rangle \in \R$ and $\ov{u} = \exp( q \ov{K}^\top) \in \R^{n-r} $ and $\ov{\alpha} = \langle \ov{u}, {\bf 1}_{n-r} \rangle \in \R$ as defined in Definition~\ref{def:top_r_softmax_attention}.

Then, we have
\begin{align*}
    & ~ \| \mathsf{Attn}_s(q, K, V) -  \wh{\mathsf{Attn}}_s(q, K, V) \|_{\infty} \\
    = & ~ \| (\wh{\alpha} + \ov{\alpha})^{-1} (\wh{u} \wh{V} + \ov{u} \ov{V} ) - \wh{\alpha}^{-1} \wh{u} \wh{V} \|_{\infty}\\
    \le & ~ \| ( (\wh{\alpha} + \ov{\alpha})^{-1} - \wh{\alpha}^{-1})  \wh{u} \wh{V}  \|_{\infty} + \| (\wh{\alpha} + \ov{\alpha})^{-1} \ov{u} \ov{V}  \|_{\infty}\\
    \le & ~ |( \wh{\alpha} + \ov{\alpha})^{-1} - \wh{\alpha}^{-1}| \cdot \| \wh{u} \|_1 \cdot \|\wh{V}  \|_{\infty} + (\wh{\alpha} + \ov{\alpha})^{-1} \cdot \|  \ov{u}\|_1 \cdot \| \ov{V}  \|_{\infty}\\
    = & ~ (\wh{\alpha}^{-1} - ( \wh{\alpha} + \ov{\alpha})^{-1})  \cdot \wh{\alpha}  \cdot \|\wh{V}  \|_{\infty} + (\wh{\alpha} + \ov{\alpha})^{-1} \cdot \ov{\alpha} \cdot \| \ov{V}  \|_{\infty}\\
    \le & ~ (\wh{\alpha}^{-1} - ( \wh{\alpha} + \ov{\alpha})^{-1})  \cdot \wh{\alpha}  \cdot \|V  \|_{\infty} + (\wh{\alpha} + \ov{\alpha})^{-1} \cdot \ov{\alpha} \cdot \| V  \|_{\infty}\\
    = & ~ 2(\wh{\alpha} + \ov{\alpha})^{-1} \cdot \ov{\alpha} \cdot \| V  \|_{\infty}
    \\
    = & ~ 2 \alpha^{-1} \cdot \ov{\alpha} \cdot \| V  \|_{\infty},
\end{align*}
where the first step is by Definition~\ref{def:top_r_softmax_attention}, the second step is by triangle inequality, the third step is by $\|uV\|_\infty \le \|u\|_1 \cdot \|V\|_\infty$ for any vector $u$ and conformable matrix $V$, and the fourth step is by definition of $\wh{\alpha}$ and $\ov{\alpha}$, i.e.,  $\wh{\alpha} = \langle \wh{u}, {\bf 1}_{r} \rangle = \|\wh{u}\|_1 $ (note that each entry of $\wh{u}$ is positive), the fifth step is by $\max\{ \| \wh{V}  \|_{\infty},  \| \ov{V}  \|_{\infty} \} =  \| V  \|_{\infty}$, the sixth step in by simple calculation and the last step is by $\wh{\alpha} + \ov{\alpha} = \alpha$.
\end{proof}

Building on this, we now present a more specific error analysis incorporating the massive activation property:

\begin{theorem}
[Error analysis of Softmax attention with index set, formal version of Theorem~\ref{thm:err_analysis_of_Softmax_attn_with_index_set:informal}] \label{thm:err_analysis_of_Softmax_attn_with_index_set}
If the following conditions hold:
\begin{itemize}
    \item Let $Q \in \R^{m \times d}$, $K, V \in \R^{n \times d}$ and the Softmax attention $\mathsf{Attn}_s$ be defined in Definition~\ref{def:Softmax_attention}.
    \item Let $q \in \R^d$ denote a single row of $Q \in \R^{m \times d}$. 
    \item Let $\gamma \in [0,1]$, $\beta_1 \ge \beta_2 \ge 0$. 
    \item Let 
    the Softmax attention with index set $\wh{\mathsf{Attn}}_s$ be defined as Definition~\ref{def:top_r_softmax_attention}. 
    \item Let $\mathsf{NN}(r, q, K) \subseteq [n]$ denote the indices of top-$r$ entries of $q K$.
    \item Let $R = \mathsf{NN}(n^\gamma, q, K) \subseteq [n]$, where $|R| = n^\gamma$. 
    \item Assume the query $q$ and key cache $K$ have $(\gamma, \beta_1, \beta_2)$ massive activation property.
\end{itemize}
 
Then, we can show that
\begin{align*}
    \| \wh{\mathsf{Attn}_s}(q, K, V) - \mathsf{Attn}_s(q, K, V) \|_\infty \le \frac{ 2 \| V  \|_{\infty}}{n^{\gamma + (\beta_1 - \beta_2)\cdot \|q\|_2 -1} } .
\end{align*} 
\end{theorem}

\begin{proof}
Let $\alpha, \ov{\alpha}, \wh{\alpha}$ be defined in Definition~\ref{def:top_r_softmax_attention}.
By Lemma~\ref{lem:Softmax_general_err_bound}, we have
\begin{align*}
    \| \mathsf{Attn}_s(q, K, V) -  \wh{\mathsf{Attn}}_s(q, K, V) \|_{\infty} \le  & ~ \frac{2 \ov{\alpha}}{\alpha}  \cdot \| V  \|_{\infty}.
\end{align*}

By Definition~\ref{def:massive}, we have
\begin{align*}
    \wh{\alpha} = & ~ \sum_{i\in \mathsf{NN}(n^\gamma, q, K)} \exp(\langle q, K_i \rangle) \\
    \ge & ~  \sum_{i\in \mathsf{NN}(n^\gamma, q, K)} \exp(\|q\|_2 \beta_1 \log(n))\\
    = & ~  n^{ \gamma + \beta_1 \cdot \|q\|_2},
\end{align*}
where the first step is by Definition of $\wh{\alpha}$, the second step is by Definition~\ref{def:massive} and Jensen inequality, and the last step is by simple calculation.  

We also have
\begin{align*}
    \ov{\alpha} = & ~ \sum_{i\in [n] \setminus \mathsf{NN}(n^\gamma, q, K)} \exp(\langle q, K_i \rangle) \\
    \le & ~  \sum_{i\in [n] \setminus \mathsf{NN}(n^\gamma, q, K)} \exp(\|q\|_2 \beta_2 \log(n))\\
    \le & ~ n^{1 + \beta_2 \cdot \|q\|_2},
\end{align*}
where the first step is by Definition of $\ov{\alpha}$, the second step is by Definition~\ref{def:massive}, and the last step is by simple calculation.  

Finally, we finish the proof by the fact $\wh{\alpha} + \ov{\alpha} = \alpha$.
\end{proof}




\end{document}